\documentclass[journal]{IEEEtran}
\usepackage[utf8]{inputenc}
\usepackage{enumitem,color}   
\usepackage{graphicx,epsfig}
\usepackage{amsmath,mathrsfs} 
\usepackage{amssymb,amsthm} 
\usepackage{amsfonts}
\usepackage{mathtools}
\usepackage{epstopdf}
\usepackage{ifthen}
\usepackage{bbm}
\graphicspath{{eps/}}
\usepackage{amsthm}
\usepackage{pgfplots}
 
\newtheorem{theorem}{Theorem}
\newtheorem{remark}{Remark} 
\newtheorem{lemma}{Lemma} 
\newtheorem{proposition}{Proposition}

\newcommand{\R}{ \mathbb{R}}

\newcommand{\eqdef}{\stackrel{\vartriangle}{=}}

\def\V#1{{\boldsymbol{#1}}}         % vectors
\def\M#1{{\bf{#1}}}  % matrices
\title{   Deep Neural Networks with Trainable  Activations and Controlled Lipschitz Constant }
 
\author{ Shayan~Aziznejad, Harshit~Gupta, Joaquim~Campos,  and~Michael~Unser,~\IEEEmembership{Fellow,~IEEE}
\thanks{The authors are with the Biomedical Imaging Group, École polytechnique fédérale de Lausanne, 1015 Lausanne,
Switzerland. e-mails:   \{shayan.aziznejad, harshit.gupta, joaquim.campos, michael.unser\}@epfl.ch}
\thanks{An early version of this work has been presented at the IEEE International Conference on Acoustics, Speech and Signal Processing (ICASSP), Brighton, United Kingdom, May 2019 \cite{aziznejad2019deep}.}
\thanks{This work was supported in part by   the Swiss National Science Foundation, Grant 200020\_184646/1 and in part by the European Research Council (H2020-ERC Project GlobalBioIm) under Grant 692726. }}

\begin{document}

\maketitle
\thispagestyle{plain}
\pagestyle{plain}

\begin{abstract}
  We introduce a variational framework to learn  the activation functions of  deep neural networks. Our aim is to increase the capacity of the network while controlling an upper-bound of the actual Lipschitz constant of  the  input-output relation. To that end,   we first establish a global bound for the Lipschitz constant of neural networks. Based on the obtained bound, we then formulate a variational problem for learning activation functions. Our variational problem is infinite-dimensional and is not computationally  tractable. However, we prove that there always exists a solution that has continuous and piecewise-linear (linear-spline) activations. This reduces the original  problem to a finite-dimensional minimization where an $\ell_1$ penalty on the parameters of the activations favors the learning of sparse nonlinearities. We numerically compare our scheme with standard ReLU network and   its variations, PReLU and  LeakyReLU and we empirically demonstrate the practical aspects of our framework.
\end{abstract}


\section{Introduction}
\label{Sec:intro}

In supervised learning, the goal is to approximate an unknown mapping from a set of  noisy samples \cite{Bishop2006}.  Specifically, one aims at determining the function $f:\mathbb{R}^d\rightarrow\mathbb{R}^{d'}$, given a dataset of size $M$ that consists  of pairs of the form $(\boldsymbol{x}_m,\boldsymbol{y}_m)$  such that $\boldsymbol{y}_m\approx f(\boldsymbol{x}_m)$ for $m=1,2,\ldots,M$, without over-fitting. 

In the scalar case $d'=1$, a classical formulation of this problem is through the minimization  
\begin{equation}\label{Pb:RKHSLearning}
\min_{f\in \mathcal{H}(\mathbb{R}^d)} \left( \sum_{m=1}^M \mathrm{E}\left( y_m, f(\V x_m)\right) + \lambda\|f\|_{\mathcal{H}}^2\right), 
\end{equation}
where $\mathcal{H}(\mathbb{R}^d)$  is a reproducing-kernel Hilbert space  (RKHS), $\mathrm{E}:\mathbb{R}\times \mathbb{R}\rightarrow\mathbb{R}$ is an arbitrary convex loss function,    and $\lambda$ is a positive constant that controls the regularity of the model \cite{aronszajn1950theory,Wahba1990}. Although this is an infinite-dimensional problem over a Hilbert space, the kernel representer theorem  \cite{kimeldorf1971some,scholkopf2001generalized} states that the solution of \eqref{Pb:RKHSLearning} is unique and admits the parametric form   
\begin{equation}\label{Eq:RKHSSOL}
f(\V x) = \sum_{m=1}^M a_m \mathrm{k}(\V x,\V x_m),
\end{equation}
where $\mathrm{k}(\cdot,\cdot)$ is the unique reproducing kernel of $\mathcal{H}(\mathbb{R}^d)$.  Expansion \eqref{Eq:RKHSSOL} is the key element of kernel-based algorithms in machine learning, including the    framework of support-vector machines  \cite{Vapnik1998}. It also reveals the intimate link between kernel methods,  splines, and  radial basis functions  \cite{ Poggio1990b, Scholkopf1997,girosi1993priors}. 

During the past decade, there has been an increasing interest in deep-learning methods as they   outperform  kernel-based schemes in a variety of tasks such as image classification \cite{Krizhevsky2012}, inverse problems \cite{jin2017deep}, and segmentation \cite{Ronneberger2015}.   The main idea is to replace the kernel expansion \eqref{Eq:RKHSSOL} by a parametric deep neural network that is a repeated composition of affine mappings intertwined with pointwise nonlinearities (a.k.a.\ neuronal activations)  \cite{LeCun2015,Goodfellow2016}.   The challenge then is to optimize the parameters of this model by minimizing  a (typically non-convex) cost function. 

The classical choice for the activation is the  sigmoid function due to its biological interpretation and universal approximation property \cite{cybenko1989approximation}.   However, neural networks with sigmoidal activations suffer from vanishing gradients   which essentially makes training difficult and slow. This stems from the fact that the Sigmoid function is bounded   and     horizontally asymptotic at large positive and negative values. Currently, the preferred activations  are  rectified linear unit $\mathrm{ReLU}(x)=\max(x,0)$ \cite{Glorot2011}  and  its variants such as LeakyReLU, defined as   $\mathrm{LReLU}(x)=\max(x,ax)$ for some $a\in (0,1)$ \cite{maas2013rectifier}.  ReLU-based activations have a wide range, which prevents the network from having vanishing gradients.

Neural networks with ReLU activations have been considered thoroughly in the literature  \cite{LeCun2015}. Their  input-output relation is a continuous piecewise-linear (CPWL) mapping \cite{Montufar2014}. Interestingly, the converse of this result  also holds: Any CPWL function can be represented by a deep ReLU neural network \cite{arora2016understanding}.   One can also interpret a ReLU activation  as a linear spline with one knot. This observation allows one to interpret  deep ReLU networks   as hierarchical splines \cite{Poggio2015}. In addition, Unser showed that linear-spline activations are  optimal in the sense that they have a minimal second-order total variation and, hence, are maximally regularized \cite{Unser2018}; this also provides a   variational justification for ReLU-based activations.  

Although the ReLU networks are favorable, both from  a theoretical and practical point of view, one may   want to go even farther and learn  the activation functions as well. The minimal attempt is to  learn the parameter $a$ in LeakyReLU activations, which is known as the parametric ReLU (PReLU) \cite{he2015delving}.  More generally, one can consider a parametric form for the activations and learn the parameters in the training step. There is a rich literature on the learning of activations represented by splines, a parametric form   characterized by optimality and universality \cite{de1978practical,unser1999splines}. Examples are  perceptive B-splines \cite{lane1991multi},  Catmull-Rom cubic splines \cite{vecci1998learning,guarnieri1999multilayer}, and  adaptive piecewise linear splines  \cite{agostinelli2014learning}, to name a few. 

In theoretical analyses of deep neural networks, the Lipschitz-continuity of the network and the control of its regularity is of great importance and is crucial in  several schemes   of  deep learning, for example in   Wasserstein GANs  \cite{arjovsky2017wasserstein}, in providing compressed sensing type guarantees for generative models \cite{bora2017compressed}, in showing the convergence of CNN-based projection algorithms to solve inverse problems
\cite{gupta2018cnn}, and  in understanding the  generalization property of deep neural networks   \cite{neyshabur2017exploring}. Moreover,  the Lipschitz regularity drives the stability of neural networks, a matter that has been tackled recently \cite{moosavi2016deepfool,fawzi2017robustness,antun2019instabilities}.

In this paper, we propose a variational framework to learn the activation functions with the motivation of increasing  the capacity of the network while controlling its Lipschitz regularity.    To that end, we first provide a global  bound for the Lipschitz constant of the  input-output relation of   neural networks that have second-order bounded-variation activations. Based on the minimization of this bound, we propose an optimization scheme in which we learn the linear weights and the activation functions jointly.   We show that there   always exists a global solution of our proposed minimization made of linear spline activations. We also demonstrate that our proposed regularization has a sparsity-promoting effect on the parameters of the spline activations. Let us remark that our regularization, which is based on an upper-bound, does not ensure that the actual Lipschitz constant of the neural network is minimized---it only prevents it from exceeding a certain range.

Our framework is inspired  from \cite{Unser2018} and  brings in the following new elements:
\begin{itemize}
\item After a slight modification of  the regularization term that was proposed in \cite{Unser2018}, we  identify  a global bound for the Lipschitz constant of the network   (see Theorem \ref{Thm:LipschitzBV}). 
\item We prove the {\it existence} of a linear-spline solution in our framework (see Theorem \ref{Thm:Main} and the discussion after).
\item By providing  numerical examples, we show how to take advantage of  our main results to improve the expressivity of neural networks. This is of practitioner's relevance, as our activation learning module can be used to replace classical activation functions  like ReLU and its variants.  
\end{itemize}

The paper is organized as follows: In Section \ref{Sec:Prelim}, we provide  mathematical preliminaries. In Section \ref{Sec:BV2}, we discuss the properties of neural networks that have second-order bounded-variation activations and  provide a global bound for their Lipschitz constant. We then introduce our variational formulation  and  study its solutions in Section \ref{Sec:Activ}. In Section \ref{Sec:Numerical}, we  illustrate our framework with numerical examples.

\section{Preliminaries}
\label{Sec:Prelim}
 The Schwartz space of smooth and rapidly decaying functions is denoted by $\mathcal{S}(\mathbb{R})$. Its continuous dual $\mathcal{S}'(\mathbb{R})$ is the space of tempered distributions \cite{rudin1991functional}.  The space of continuous functions that vanish at infinity  is denoted by $ \mathcal{C}_0(\mathbb{R})$. It is a Banach space equipped with the supremum norm $\|\cdot\|_{\infty}$ and  is indeed the closure of $\mathcal{S}(\mathbb{R})$ with this norm. Its continuous dual  is the space of Radon measures  $\mathcal{M}(\mathbb{R})$ that is also a Banach space with the total-variation norm defined as \cite{Rudin1987}
\begin{equation}\label{Eq:TV}
\|w\|_{\mathcal{M}} \eqdef \sup_{\substack{\varphi\in\mathcal{S}(\mathbb{R}) \\ \|\varphi\|_{\infty}=1}} \langle w,\varphi \rangle.
\end{equation}
The Banach space  $(\mathcal{M}(\mathbb{R}),\|\cdot\|_{\mathcal{M}})$  is a generalization of $(L_1(\mathbb{R}),\|\cdot\|_{L_1})$, in the sense that $L_1(\mathbb{R})\subseteq \mathcal{M}(\mathbb{R})$ and, for any $f\in L_1(\mathbb{R})$, the relation $\|f\|_{L_1} = \| f\|_{\mathcal{M}}$ holds. However, it is larger than $L_1(\mathbb{R})$. For instance, it  contains the shifted Dirac impulses $\delta(\cdot-\boldsymbol{x}_0)$ with $\|\delta(\cdot-\boldsymbol{x}_0)\|_{\mathcal{M}}=1$, for all $\boldsymbol{x}_0\in\mathbb{R}$ that are not included in $L_1(\mathbb{R})$. 

The space of functions with second-order bounded variations is denoted by $\mathrm{BV}^{(2)}(\mathbb{R})$ and is defined as
\begin{equation}\label{Eq:BV2}
\mathrm{BV}^{(2)}(\mathbb{R})= \{ f \in \mathcal{S}'(\mathbb{R}): \quad  \|\mathrm{D}^2 f \|_{\mathcal{M}} < +\infty\},
\end{equation}
where $\mathrm{D}:\mathcal{S}'(\mathbb{R})\rightarrow\mathcal{S}'(\mathbb{R})$ is the generalized  derivative operator \cite{unser2014introduction}.  Let us mention that the second-order total variation $\mathrm{TV}^{(2)}(f)\eqdef  \|\mathrm{D}^2 f \|_{\mathcal{M}}  $ is only a semi-norm in this space, since the null space of the linear operator  $\mathrm{D}^2$ is nontrivial and    consists  of degree-one polynomials (affine mappings in $\mathbb{R}$). However, it can become a {\it bona fide} Banach space with the $\mathrm{BV}^{(2)} $ norm
\begin{equation}\label{Eq:BV2norm}
\|f\|_{\mathrm{BV}^{(2)} } \eqdef {\rm TV}^{(2)}(f) + |f(0)| +   |f(1)|. 
\end{equation}
 This space has been extensively studied in \cite{Unser2018} and in a more general setting in \cite{unser2017splines}.  We summarize some of its important properties in  Appendix \ref{App:BV2Space}. 

Given generic  Banach spaces $(\mathcal{X},\|\cdot\|_{\mathcal{X}})$ and  $(\mathcal{Y},\|\cdot\|_{\mathcal{Y}})$, a  function $f:\mathcal{X}\rightarrow\mathcal{Y}$ is said to be Lipschitz-continuous if  there exists a finite constant $C>0$ such that  
\begin{equation}\label{Eq:Lipschitz}
\|f(x_1)-f(x_2)\|_{\mathcal{Y}} \leq C \|x_1-x_2\|_{\mathcal{X}}, \quad \forall x_1,x_2 \in \mathcal{X}.
\end{equation}  
The minimal value of $C$ is  called the Lipschitz constant of $f$. 

In this paper, we   consider fully connected feed forward neural networks. An $L$-layer neural network $\mathbf{f}_{\mathrm{deep}}:\mathbb{R}^{N_0}\rightarrow\mathbb{R}^{N_L}$ with the layer descriptor $(N_0,N_1,\ldots,N_L)$ is the  composition of   the vector-valued functions $\mathbf{f}_{l}:\mathbb{R}^{N_{l-1}} \rightarrow\mathbb{R}^{N_{l}}$ for $l=1,\ldots,L$ as   
\begin{equation}\label{Eq:DNN}
\M f_{\mathrm{deep}} :\R^{N_0} \to \R^{N_{L}}: \V x \mapsto \M f_L \circ \dots \circ \M f_1(\V x).
\end{equation}
Each vector-valued function $\mathbf{f}_{l}$ is a layer of the neural network $\mathbf{f}_{\mathrm{deep}}$ and consists of two elementary operations: linear transformations and point-wise nonlinearities.  In other words, for the  $l$th layer, there exists weight vectors $\mathbf{w}_{n,l}\in\mathbb{R}^{N_{l-1}}$ and nonlinear functions (activations) $\sigma_{n,l}:\mathbb{R}\rightarrow\mathbb{R}$ for $n=1,2,\ldots,N_l$ such that 
\begin{equation}
\label{Eq:DNNLayer}
\M f_l(\V x) = \left( \sigma_{1,l} (\M w_{1,l}^T \V x) ,\sigma_{2,l} (\M w_{2,l}^T \V x) ,\ldots,\sigma_{N_l,l} (\M w_{N_l,l}^T \V x) \right).
\end{equation}
  
\begin{figure}[t]
\begin{minipage}{1.0\linewidth}
  \centering
  \centerline{\includegraphics[width=\linewidth]{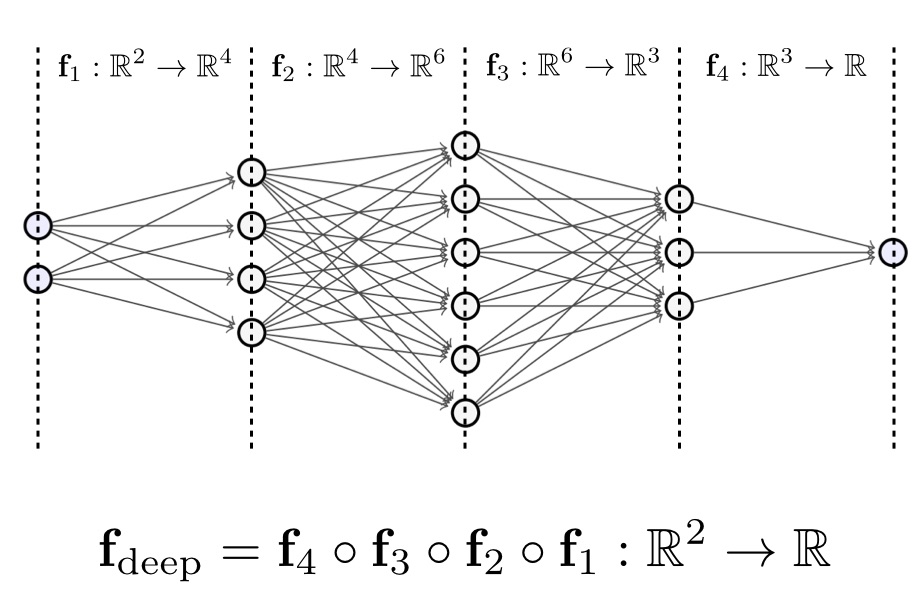}}%  \vspace{2.0cm}
  \caption{Schematic view of a neural network with the layer descriptor $(2,4,6,3,1)$. Each layer   consists of linear weights (arrows) and point-wise nonlinearities (circles).}
  \label{BPFig} \medskip
\end{minipage}
\end{figure}

One can also consider  an alternative representation of the $l$th layer by defining the matrix $\mathbf{W}_{l} = \begin{bmatrix} \mathbf{w}_{1,l} &\mathbf{w}_{2,l} &\cdots &\mathbf{w}_{N_l,l} 
\end{bmatrix}^T$ and the vector-valued nonlinear function $\boldsymbol{\sigma}_l:\mathbb{R}^{N_{l}}\rightarrow\mathbb{R}^{N_l}$ as the mapping
\begin{equation}
(x_1,\ldots,x_{N_{l}})\mapsto (\sigma_{1,l}(x_1),\sigma_{2,l}(x_2),\ldots,\sigma_{N_l,l}(x_{N_l})). 
\end{equation} 
With this notation,  the $l$th layer has simply the form $\mathbf{f}_{l} = \boldsymbol{\sigma}_{l} \circ \mathbf{W}_{l}$. 

Lastly, for any $p\in [1,+\infty)$, we define the $(\mathrm{BV}^{(2)},p)$-norm of the nonlinear layer $\boldsymbol{\sigma}_{l}$ as 
\begin{equation}
\|\boldsymbol{\sigma}_{l}\|_{{\rm BV}^{(2)},p} = \left( \sum_{n=1}^{N_l} \|\sigma_{n,l}\|_{{\rm BV}^{(2)}}^p \right)^{\frac{1}{p}}.
\end{equation}

\section{Second-Order Bounded-Variation Activations}
\label{Sec:BV2}
We   consider activations from the space of second-order bounded-variation functions $\mathrm{BV}^{(2)}(\mathbb{R})$. This ensures that the corresponding neural network satisfies several desirable properties  which we discuss in this section. The key feature of these activations is their Lipschitz continuity, as stated in Proposition \ref{BVLipschitz}. The proof is provided in Appendix \ref{App:BVLip}. 
\begin{proposition}\label{BVLipschitz}
Any function with second-order bounded variation is Lipschitz-continuous. Specifically, for any function $\sigma \in {\rm BV}^{(2)}(\R)$ and any $x_1,x_2\in \mathbb{R}$, we have that 
\begin{align}
\label{Eq:Lipchitz}
|\sigma(x_1)-\sigma(x_2)| \le \|\sigma \|_{{\rm BV}^{(2)}} |x_1-x_2|.
\end{align} 
\end{proposition}

Lipschitz functions are known to be continuous and  differentiable almost everywhere \cite{heinonen2005lectures}. Moreover,  in Proposition \ref{Prop:OneSidedDerivatives}, we show that any element of $\mathrm{BV}^{(2)}(\mathbb{R})$ has well-defined right and left derivatives at any point.  This is an important property for activation functions, since it is the minimum requirement for performing gradient-based algorithms that take advantage of the celebrated back-propagation scheme  in the training step  \cite{rumelhart1988learning}.

\begin{proposition}\label{Prop:OneSidedDerivatives}
For any function $\sigma \in \mathrm{BV}^{(2)}(\mathbb{R})$ and any $x_0 \in \mathbb{R}$, the left and right derivatives of $\sigma$ at the point $x=x_0$ exist and are finite.   
\end{proposition}
The proof of Proposition \ref{Prop:OneSidedDerivatives} is available in   Appendix \ref{App:OneSideDer}. Let us mention that Lipschitz functions in general do  not have one-sided derivatives at all points; it is a property that is specific of ${\rm BV}^{(2)}$ functions. As an example, consider the  function $f:\mathbb{R}\rightarrow\mathbb{R}$ with 
\begin{equation}
f(x)=\begin{cases}    x \sin(\log(x)), & x>0  \\
0, & x\leq 0. 
\end{cases}
\end{equation} 
One readily   verifies that $f$ is Lipschitz-continuous with the constant   $C=\sqrt{2}$. However, for  positive values of $h$, the function $\frac{f(h)}{h}= \sin(-\log(h))$ oscillates between $(-1)$ and $1$ as  $h$ goes to zero. Hence, $f$ does  not  have a right derivative at the point $x_0=0$.

In Theorem \ref{Thm:LipschitzBV}, we prove that any neural network with activations from $\mathrm{BV}^{(2)}(\mathbb{R})$ specifies a Lipschitz-continuous input-output relation. Moreover, we provide an upper-bound for its Lipschitz constant.  The proof can be found in  Appendix \ref{App:NNLip}. 
\begin{theorem}\label{Thm:LipschitzBV}
Any feed forward fully connected deep neural network $\mathbf{f}_{\mathrm{deep}}:\mathbb{R}^{N_0}\rightarrow\mathbb{R}^{N_L}$ with  second-order bounded-variation activations $\sigma_{n,l}\in \mathrm{BV}^{(2)}(\mathbb{R})$ is Lipschitz-continuous. Moreover, if we consider the $\ell_p$  for $p\in [1,\infty]$  topology in the input and output spaces,  the neural network satisfies the global  Lipschitz bound 
\begin{align}\label{LipschitzFinalBound}
\left\| \M f_{\rm deep}(\V x_1) - \M f_{\rm deep}(\V x_2)\right\|_p & \le  C \|\V x_1-\V x_2\|_p 
\end{align}
 for all   $\V x_1 , \V x_2 \in \mathbb{R}^{N_0}$, where 
\begin{equation}\label{LipschitzConstant}
C = \left( \prod_{l=1}^L  \|\mathbf{W}_{l}\|_{q,\infty} \right). \left( \prod_{l=1}^L \|\boldsymbol{\sigma}_l\|_{\mathrm{BV}^{(2)},p}\right),
\end{equation}
$q \in [1,\infty]$ is such that $\frac{1}{p}+\frac{1}{q}=1$ and   $ \|\mathbf{W}_{l}\|_{q,\infty} = \max_{n} \|\mathbf{w}_{n,l}\|_{q}$ is  the  mixed norm  ($\ell_q-\ell_{\infty}$)
of the $l$th linear layer. 
\end{theorem}
When the standard  Euclidean topology is assumed for the input and output spaces, Proposition \ref{Prop:LipEuc} provides an alternative bound for the Lipschitz constant of the neural network. The proof of Proposition \ref{Prop:LipEuc} is available in Appendix \ref{App:LipEuc}.
\begin{proposition}\label{Prop:LipEuc}
Let $\mathbf{f}_{\mathrm{deep}}:\mathbb{R}^{N_0}\rightarrow\mathbb{R}^{N_L}$ be a fully connected feed forward neural network with activations selected from $\mathrm{BV}^{(2)}(\mathbb{R})$.   For all   $\V x_1 , \V x_2 \in \mathbb{R}^{N_0}$ we have that 
\begin{align}\label{LipschitzFinalBound2}
\left\| \M f_{\rm deep}(\V x_1) - \M f_{\rm deep}(\V x_2)\right\|_2& \le  C_E \|\V x_1-\V x_2\|_2,
\end{align}
 where 
\begin{equation}\label{LipschitzConstant2}
C_E = \left( \prod_{l=1}^L  \|\mathbf{W}_{l}\|_{F} \right). \left( \prod_{l=1}^L \|\boldsymbol{\sigma}_l\|_{\mathrm{BV}^{(2)},1}\right).
\end{equation}
\end{proposition}
\begin{remark}\label{Rem:lpNorm}
In Proposition \ref{Prop:LipEuc}, it is possible to replace the $\ell_1$ outer-norm of the nonlinear layers by  $\ell_p$ for any $p\in[1,\infty]$. This is due to the equivalence of norms in finite-dimensional vector-spaces. In general, such replacements come  at the cost of multiplying the Lipschitz bound by a constant. In the special case $p=2$, no constant is required and we achieve an even tighter bound (following \eqref{Eq:BoundEuclid} in Appendix \ref{App:LipEuc}). However,  we favour $\ell_1$ due to its globally sparsifying effect (see Section \ref{Subsec:L1vsL2} for a numerical illustration). 
\end{remark}
Proposition \ref{Prop:LipEuc} will take  a particular relevance in Section \ref{Sec:Activ}, where \eqref{LipschitzConstant2} will allow us to design a joint-optimization problem to learn  the linear weights and activations. Interestingly, the proposed minimization is compatible  with  the use of weight decay  \cite{krogh1992simple} in the training of neural networks (see \eqref{Eq:DeepCost} with $\mathrm{R}(\mathbf{W}) = \|\mathbf{W}\|_F^2$).  
\section{Learning Activations}
\label{Sec:Activ}
In this section, we propose a novel variational formulation to learn Lipschitz activations in a deep neural network. We select $\mathrm{BV}^{(2)}(\mathbb{R})$ as our search space to ensure the Lipschitz continuity of the input-output relation of the global network (see Theorem \ref{Thm:LipschitzBV}). 

Similarly to the RKHS theory, the (weak*) continuity of the sampling functional  is needed to guarantee the well-posedness of the learning problem.  This is  stated   in Theorem \ref{Thm:weakstarNN} whose proof can be found in Appendix \ref{App:NNweakstar}.  (We define the notion of weak*-convergence of neural networks in Appendix \ref{App:BV2Space}.)
 \begin{theorem}\label{Thm:weakstarNN}
For any $\boldsymbol{x}_0\in\mathbb{R}^{N_0}$, the sampling functional  $\delta_{\boldsymbol{x}_0}: \mathbf{f}_{\rm deep} \mapsto \mathbf{f}_{\rm deep} (\boldsymbol{x}_0)$ is weak*-continuous in the space of neural networks with second-order bounded-variation activations.  
\end{theorem}
Given the data-set $(X,Y)$ of size $M$ that consists in the pairs $(\V x_m, \V y_m)\in \mathbb{R}^{N_0}\times\mathbb{R}^{N_L}$ for $m=1,2,\ldots,M$, we then consider the following cost functional
\begin{align} \label{Eq:DeepCost}
\mathcal{J}(\mathbf{f}_{\rm deep};X,Y) = & \sum_{m=1}^M \mathrm{E}\big(\V y_m,\M f_{\rm deep}(\V x_m)\big) +  \sum_{l=1}^L \mu_l   \mathrm{R}_l( \M W_{l}) \nonumber \\ &
\mbox{}+ \sum_{l=1}^L \lambda_l  \|\boldsymbol{\sigma}_{l}\|_{\mathrm{BV}^{(2)},1},
\end{align}
where $\mathbf{f}_{\rm deep}$ is a neural network with linear layers $\mathbf{W}_{l}$ and nonlinear layers $\boldsymbol{\sigma}_{l}=(\sigma_{1,l},\ldots,\sigma_{N_l,l})$, as specified in \eqref{Eq:DNN} and \eqref{Eq:DNNLayer},  $\mathrm{E}(\cdot,\cdot)$ is an arbitrary loss function, and  $ \mathrm{R}_l:\mathbb{R}^{N_{l}\times N_{l-1}}\rightarrow\mathbb{R}$ is a regularization functional for the linear weights of the $l$th layer. The standard choice for weight regularization is  the Frobenius norm $\mathrm{R}(\mathbf{W}) = \|\mathbf{W}\|_F^2$, which corresponds to weight decay scheme  in deep learning.  Finally, the positive constants $\mu_l,\lambda_{l}>0$ balance the regularization effect in the training step.

Theorem \ref{Thm:Main} states that, under some natural conditions, there   always exists a solution of \eqref{Eq:DeepCost} with continuous piecewise-linear activation functions, which we refer to as a {\it deep-spline} neural network.  The proof  is given in Appendix \ref{App:Main}. 

\begin{theorem}\label{Thm:Main}
Consider the training of a deep neural network via the minimization 
\begin{align}\label{Pb:DeepSpline}
 { \min_{\substack{ \mathbf{w}_{n,l}\in\mathbb{R}^{N_{l-1}}, \\ \sigma_{n,l} \in {\rm BV}^{(2)}(\R)}}} 
& \mathcal{J}(\mathbf{f}_{\rm deep};X,Y),
\end{align}
   where $\mathcal{J}(\cdot;X,Y)$ is defined in \eqref{Eq:DeepCost}. Moreover, assume that  the loss function  $\mathrm{E}(\cdot,\cdot)$ is  proper,  lower semi-continuous, and  coercive. Assume that the regularization functionals $\mathrm{R}_{l}$ are continuous, and  coercive. Then,  there   always exists a solution $\mathbf{f}_{\mathrm{deep}}^*$ of  \eqref{Eq:DeepCost} with   activations $\sigma_{n,l}$ of the form 
\begin{equation}\label{Eq:ActivationDeepSpline}
\sigma_{n,l}(x)= \sum_{k=1}^{K_{n,l}} a_{n,l,k} \mathrm{ReLU}(x-\tau_{n,l,k}) + b_{1,n,l} x+ b_{2,n,l},
\end{equation}
where $K_{n,l} \leq M$ and, $a_{n,l,k},\tau_{n,l,k},b_{\cdot,n,l}\in\mathbb{R}$ are adaptive parameters.  
\end{theorem}
 Theorem \ref{Thm:Main}   suggests an optimal ReLU-based parametric to learn  activations. This is a remarkable property as it translates the original infinite-dimensional problem  \eqref{Pb:DeepSpline} into a finite-dimensional parametric optimization, where one only needs to determine the ReLU weights $a_{n,l,k}$ and positions $\tau_{n,l,k}$ together with the affine terms  $b_{1,n,l},b_{2,n,l}$.  Let us also mention that the baseline ReLU network  and its variations (PReLU, LeakyReLU) are all included in this scheme as special cases of an activation of the form \eqref{Eq:ActivationDeepSpline} with $K=1$. 

A similar result  has been shown in the deep-spline representer theorem of Unser in \cite{Unser2018}. However, there are  three fundamental differences. Firstly, we relax the assumption of having normalized weights due to the practical considerations and the optimization challenges it brings.  Secondly, we slightly modify the regularization functional that enables us to control the global Lipschitz constant of the neural network. Lastly, we show the existence of a minimizer in our proposed variational formulation that is, to the best of our knowledge, the first result of existence in this framework. 

We remark that the choice of our regularization restrains the coefficients  $b_{1,n,l}$ and $b_{2,n,l}$ from taking high values. This   enables  us to obtain the global bound \eqref{LipschitzConstant} for the Lipschitz constant of the network, as opposed to the framework of \cite{Unser2018}, where only a semi-norm has been used for the regularization. The payoff is that, in \cite{Unser2018}, the activations have at most $(M-2)$  knots, which are the junctions between the consecutive linear pieces of a piecewise linear function, while our bound is $K_{n,l}\leq M$. This is the price to pay for controlling the Lipschitz regularity of the network. However,  this is inconsequential   in practice since there are usually much fewer knots than data points, because of the regularization penalty. The latter is   justified through the computation of the ${\rm BV}^{(2)}$ norm of an activation of the form \eqref{Eq:ActivationDeepSpline}. It yields
\begin{equation}\label{Eq:BV2Activ}
\|\sigma_{n,l}\|_{\mathrm{BV}^{(2)}} = \|\boldsymbol{a}_{n,l}\|_{1} + |\sigma(1) | + |\sigma(0)|,
\end{equation} 
where $\boldsymbol{a}_{n,l}=(a_{n,l,1},\ldots,a_{n,l,K_{n,l}})$ is the vector of ReLU coefficients. This shows that the ${\rm BV}^{(2)}$-regularization imposes  an   $\ell_1$ penalty on the ReLU weights in the expansion \eqref{Eq:ActivationDeepSpline}, thus promoting sparsity   \cite{donoho2006most}. In Section \ref{Sec:Numerical}, we illustrate the sparsity-promoting effect of the $\mathrm{BV}^{(2)}$-norm with  numerical examples (see Figures \ref{Fig:Activations} and \ref{Fig:VsLambda}). 

Another interesting property of the variational formulation \eqref{Pb:DeepSpline} is the relation between the energy of consecutive linear and nonlinear layers. In Theorem \ref{Thm:BVProp}, we exploit this relation. Its proof can be found in  Appendix \ref{App:BVProp}.
\begin{theorem}\label{Thm:BVProp}
Consider Problem \eqref{Pb:DeepSpline} with the weight regularization $\mathrm{R}_l(\mathbf{W}_l) = \|\mathbf{W}_l\|_F^2$ and positive parameters $\mu_l,\lambda_l >0$ for all $l=1,2,\ldots,L$. Then, for any  of its local minima with the linear layers $\mathbf{W}_{l}$ and nonlinear layers $\boldsymbol{\sigma}_{l}$, we have that
\begin{equation}\label{BVproplayer}
\lambda_{l} \| \boldsymbol{\sigma}_{l}\|_{\mathrm{BV}^{(2)},1} = 2 \mu_{l+1} \| \mathbf{W}_{l+1}\|_F^2, \quad l=1,2,\ldots,L-1.
\end{equation}
\end{theorem}
Theorem \ref{Thm:BVProp} shows that    the regularization constants $\mu_{l}$ and $\lambda_{l}$ provide a balance between the linear and nonlinear layers.  In our experiments, we use the outcome of this theorem to determine the value of $\lambda_l$. More precisely, we select $\lambda$  such that  \eqref{BVproplayer} holds in the initial setup. This is relevant in practice as it reduces the number of hyper-parameters that one needs to tune and results  in a faster training scheme. We also show experimentally that this choice of $\lambda$ is desirable.

\section{Numerical Illustrations}
\label{Sec:Numerical}
In this section, we discuss the practical aspects of our framework and conduct    numerical experiments  in which  we compare the performance of our method to ReLU neural networks and  its variants:   LeakyReLU and PReLU activations. 

According to Theorem \ref{Thm:Main}, one can  translate the  original infinite-dimensional problem \eqref{Pb:DeepSpline} to the minimization 
\begin{align}\label{Eq:DiscretizedLearning}
 \min_{\substack{ \mathbf{w}_{n,l}\in\mathbb{R}^{N_{l-1}}, \\ {\V a}_{n,l}\in \R^{K_{n,l}}\\ b_{i,n,l} \in \R }} &  \sum_{m=1}^M \mathrm{E}\big(\V y_m,\M f_{\rm deep}(\V x_m)\big) +  \sum_{l=1}^L  \mu_l\sum_{n=1}^{N_l}  \|{\bf  w}_{n,l}\|_2^2 \nonumber \\ &
\mbox{}+ \sum_{l=1}^L  \lambda_l \sum_{n=1}^{N_l}  \left(\|\V a_{n,l} \|_1 + |\sigma_{n,l}(1)|+|\sigma_{n,l}(0)|\right),
\end{align}
where ${\bf f}_{\rm deep}$ is the global input-output mapping and $\sigma_{n,l}$ follows the parametric form  given in \eqref{Eq:ActivationDeepSpline}. The optimization hence is over a set of finitely many variables, namely,  the linear weights $\mathbf{w}_{n,l}$ and the unknown parameters of $\sigma_{n,l}$ in \eqref{Eq:ActivationDeepSpline} for each neuron $(n,l)$. The main challenge is that the number $K_{n,l}$ of ReLUs in the representation \eqref{Eq:ActivationDeepSpline} is unknown {\it a priori}. To overcome this issue, we  fix  $K_{n,l}$ to a large value (we took $K=21$ in our experiments) and rely on the sparsifying effect of ${\rm BV}^{(2)}$ regularization to promote a sparse expansion in the ReLU basis and remove  the nonessential ReLUs.  Thus, one may use the standard optimization schemes such as SGD or  ADAM to learn the activations, jointly with the other parameters of the network. At the end of training, we perform a sparsifying step in which  we annihilate some ReLU coefficients  that are selected in such a way that the training error changes less than 1 percent. 

Let us mention that the parametric form of the deep spline activation function has linear dependencies to its parameters ${\V a}_{n,l}\in \R^{K_{n,l}}$ and $b_{i,n,l} \in \R,i=1,2$. However, this does not reduce the global optimization problem \eqref{Eq:DiscretizedLearning} to the learning of a linear classifier (or regressor). Indeed, for a fixed data point $\boldsymbol{x}_m$, the quantity ${\bf f}_{\rm deep}(\boldsymbol{x}_m)$ in general has nonlinear (and even nonconvex) dependencies to the parameters ${\V a}_{n,l}\in \R^{K_{n,l}}$ and $b_{i,n,l} \in \R,i=1,2$ of the activation function of the neuron $(n,l)$.
\subsection{Setup}\label{Sec:numset}
We designed a simple experiment in which the goal is to classify points that are inside a  circle of area 2 centred at the origin. This is a classical two-dimensional supervised-learning problem, where the target function is 
\begin{equation}
\mathbbm{1}_{\text{Circle}}(x_1,x_2) = \begin{cases}1, & x_1^2+x_2^2 \leq  \frac{2}{\pi} \\ 0, & \text{otherwise}.\end{cases} 
\end{equation}
The training dataset is obtained by generating $M=1000$ random points from a uniform distribution on  $[-1,1]^2$. The points that lie inside and outside of the circle are then labeled as 1 and 0, respectively.  

To illustrate the   effect  of our proposed scheme, we consider a family of fully connected architectures with layer descriptors of the form $(2,2W,1)$, where the width parameter $W\in\mathbb{N}$ governs the complexity of the architecture. We follow  the classical choice  of using a sigmoid activation in the last layer, together with the binary cross-entropy loss  
\begin{equation}
\mathrm{E}(y,\widehat{y}) = - y \log(\widehat{y}) - (1-y) \log(1-\widehat{y}).
\end{equation}
 We take $\mu_1=\mu_2 =\mu$   and, in each scheme, we tune the single hyper-parameter $\mu>0$. 
 
 In our Lipschitz-based design,  we use  Xavier's rule \cite{glorot2010understanding} to initialize the linear weights. For the activations, we consider the simple piecewise-linear functions {\it absolute value} and {\it soft-thresholding},  defined as 
\begin{align}
f_{\rm abs}(x)& = \begin{cases} x, & x\geq 0 \\ -x, & x < 0,\end{cases} \\
 f_{\rm soft} (x) &=\begin{cases} x-\frac{1}{2}, & x\geq \frac{1}{2}\\ 0, & x \in (-\frac{1}{2},\frac{1}{2}) \\ x + \frac{1}{2}, & x\leq -\frac{1}{2}.\end{cases}
\end{align} 
We then initialize half of the activations with $f_{\rm abs}$ and the other half with $ f_{\rm soft} $. Intuitively, such initializations may allow the network to be flexible to both even and odd functions. 

Moreover, we deploy Theorem \ref{Thm:BVProp} to tune the parameter $\lambda$. A direct calculation reveals that 
\begin{equation}
\|f_{\rm abs}\|_{\rm BV^{(2)}} = 3, \qquad \|f_{\rm soft}\|_{\rm BV^{(2)}} = \frac{5}{2}.
\end{equation}
This allows us to tune $\lambda$ so that the optimality condition \eqref{BVproplayer}   holds in the initial setup. Due to the Xavier initialization,   the linear weights of the second layer have variance $\sigma^2 = 2/(2W+1)$. Therefore,  we obtain that
\begin{equation}\label{Eq:LmabdaTune}
\lambda = \frac{16}{11(2W+1)} \mu. 
\end{equation}
 For an informed comparison, we also count the total number of parameters that is used in each scheme to represent the learned function. More specifically, with the layer descriptor $(2,2W,1)$, there are $6W$ linear weights and  one bias for the last (sigmoidal) activation. In addition, there are parameters that depend  on the specific activation we are using: There is a bias parameter in ReLU and  LeakyReLU activations. In addition to bias, PReLU activation has an extra parameter (the slope in the negative part of the real line) as well and finally, the number of parameters in our scheme is adaptive and is equal to the number of active ReLUs plus the null-space coefficients  in the representation \eqref{Eq:ActivationDeepSpline}.

\subsection{Comparison with ReLU-Based Activations}
We display in Figure \ref{Fig:Circle}  the learned function $f:\mathbb{R}^2\rightarrow\mathbb{R}$ in each case.   We also disclose in Table \ref{Table:AreaClassification} the performance and the number of active parameters of each scheme. One verifies that our scheme, already in the simplest configuration with  layer descriptor $(2,2,1)$, outperforms all other methods, even when they are deployed over the richer architecture $(2,10,1)$. Moreover, there are  fewer parameters in the final representation of the target function in our scheme, as compared to the other methods. This experiment shows that the learning of activations in simple architectures is beneficial as it compensates the low capacity of the network and contributes to the generalization power of the global learning scheme. 
 
\begin{figure}[t]
\begin{minipage}{1.0\linewidth}
  \centering
  \centerline{\includegraphics[width=\linewidth]{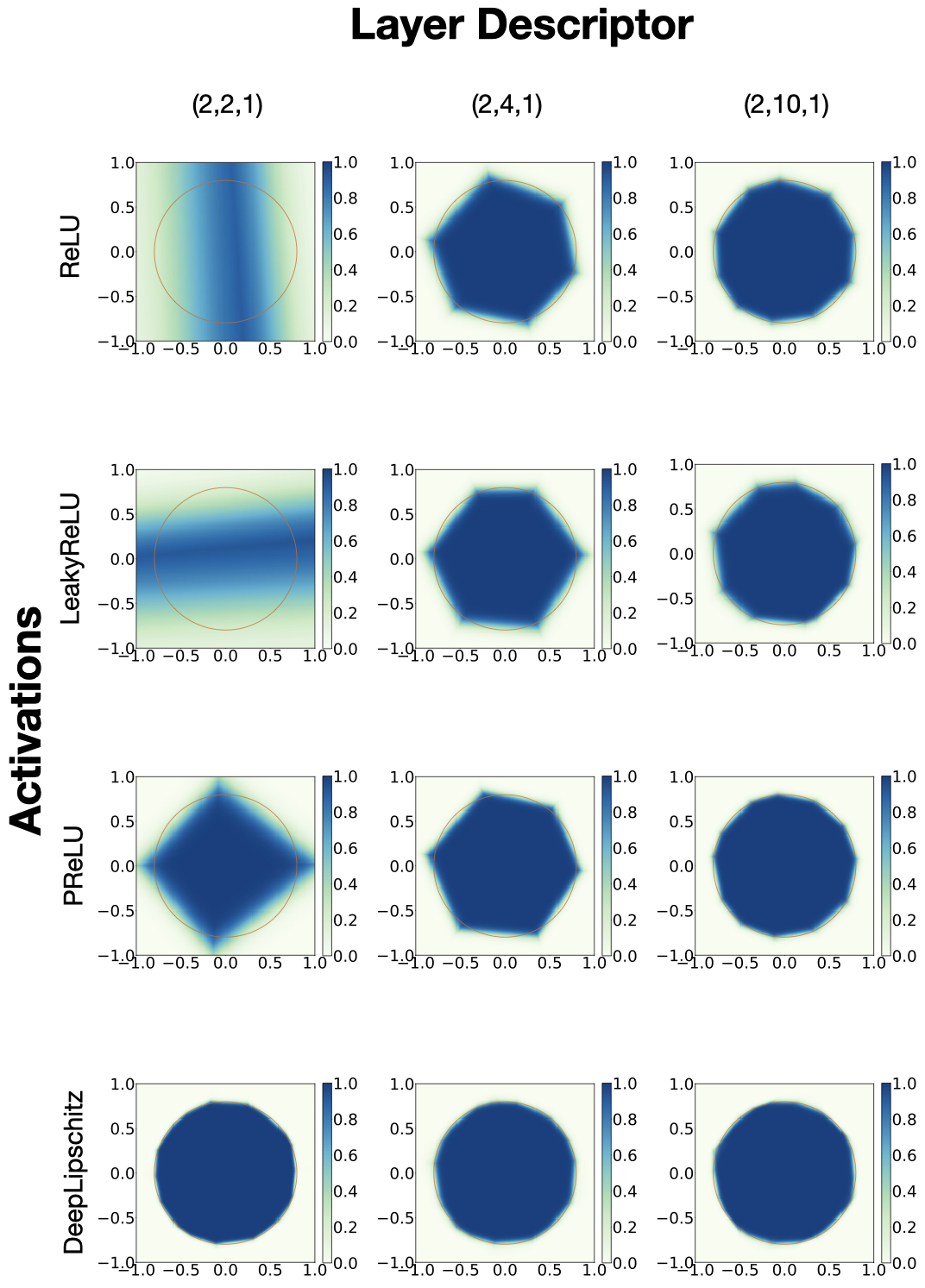}}%  \vspace{2.0cm}
  \caption{Area classification with different activations, namely ReLU, LeakyReLU, PReLU, and our proposed scheme, which we refer to as {\it Deep Lipschitz}. In each case, we consider  $W=1,2,5$  hidden neurons. }
  \label{Fig:Circle} \medskip
\end{minipage}
\end{figure} 
 In the minimal case $W=1$, we expect the network to learn parabola-type activations. This is due to the fact that the target function can be represented as 
\begin{equation}
\mathbbm{1}_{\text{Circle}}(x_1,x_2) = \mathbbm{1}_{[0, \frac{2}{\pi}]}(x_1^2+x_2^2),
\end{equation}
which is the composition of the sum of two parabolas and a threshold function. To verify this intuition, we have also plotted the learned activations for the case $W=1$ in Figure \ref{Fig:Activations}.

\subsection{Sparsity-Promoting Effect of $\mathrm{BV}^{(2)}$-Regularization}
 Despite   allowing a large number of ReLUs in the expansion \eqref{Eq:ActivationDeepSpline} ($K=21$), the learned activations (see Figure \ref{Fig:Activations}) have sparse expansion in the ReLU basis. This is due to the sparsity-promoting effect of the $\mathrm{BV}^{(2)}$-norm on the ReLU coefficients and also the thresholding step that we added at the end of training. 

\begin{figure}[t]
\begin{minipage}{1.0\linewidth}
  \centering
  \centerline{\includegraphics[width=0.95\linewidth]{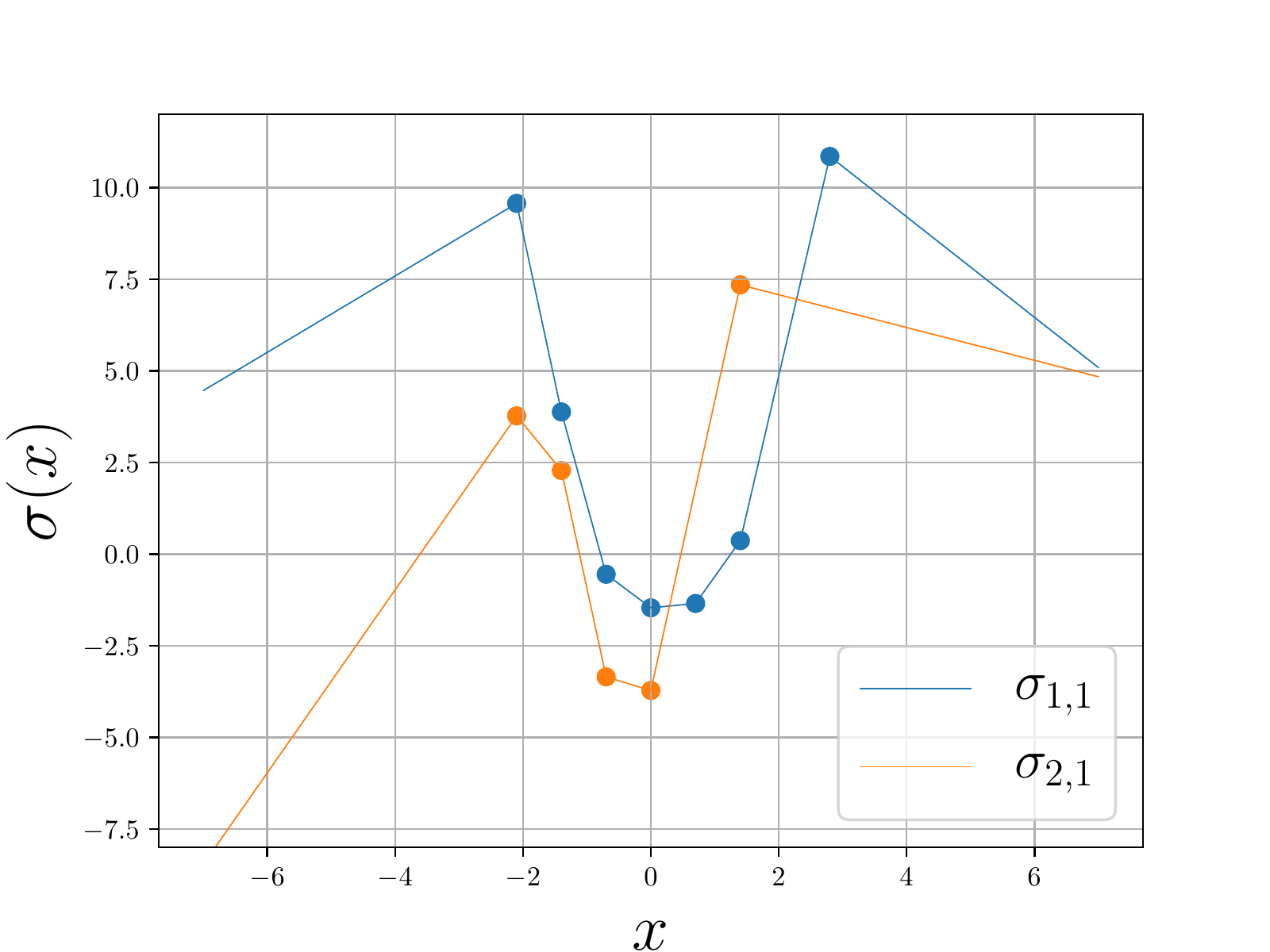}}%  \vspace{2.0cm}
  \caption{Learned activations in the area classification experiment for a simple network with layer descriptor $(2,2,1)$.}
  \label{Fig:Activations} \medskip
\end{minipage}
\end{figure}
\subsection{Effect of the Parameter $\lambda$}
To investigate the effect of the parameter $\lambda$ in our experiments, we have set the weight decay parameter to $\mu=10^{-4}$ and  plotted in Figure \ref{Fig:VsLambda} the error rate, our proposed Lipschitz bound, and the total number of active ReLUs versus $\lambda$. As expected, the sparsity and Lipschitz regulariy of the network increases with $\lambda$. Consequently,  one can control the overall regularity/complexity of the network by tuning this parameter. 

\begin{figure}[t]
\begin{minipage}{\linewidth}
\centering
\centerline{\includegraphics[width=0.95\linewidth]{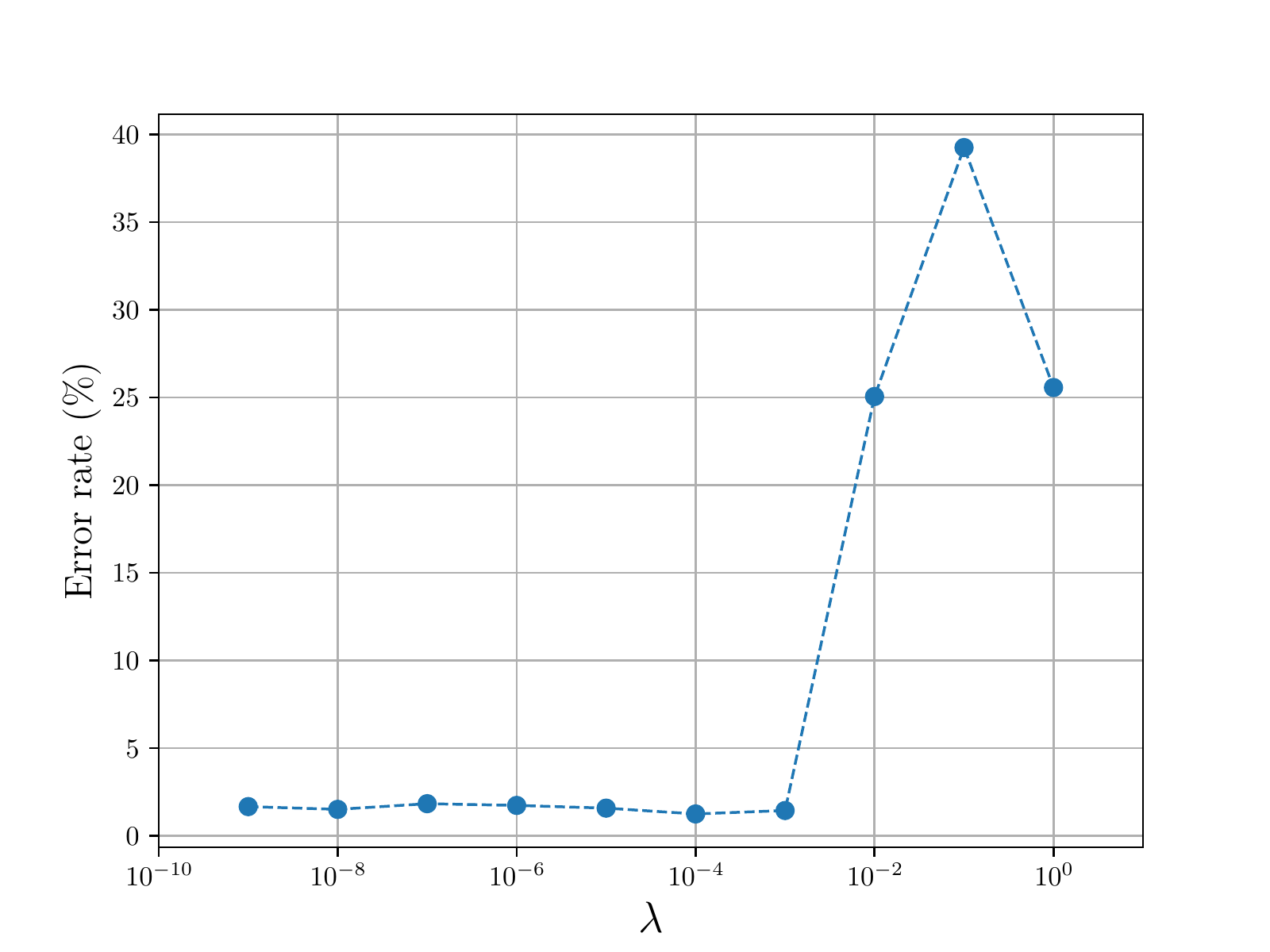}}%  \vspace{2.0cm}
\centerline{\includegraphics[width=0.95\linewidth]{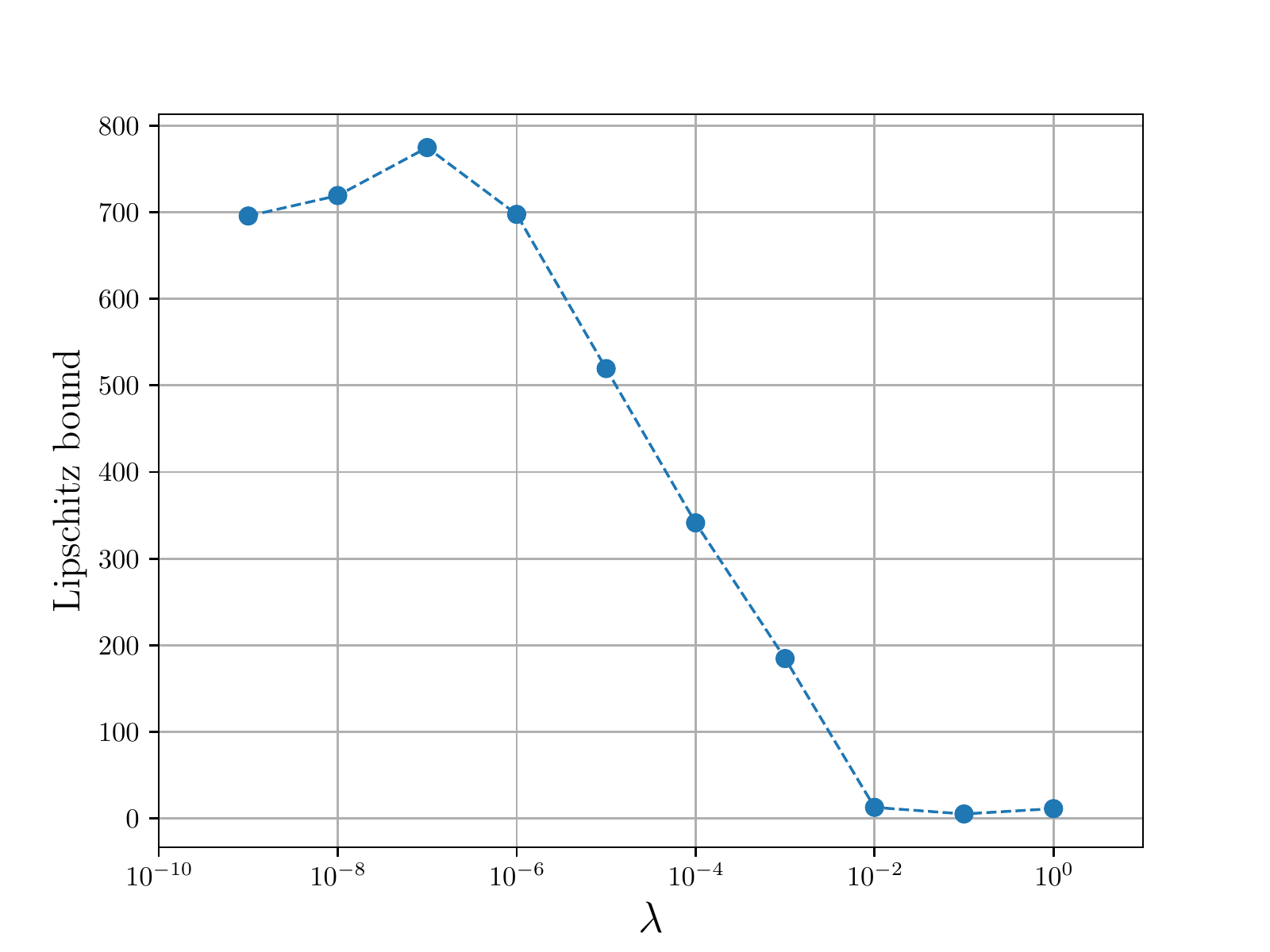}}%  \vspace{2.0cm}
\centerline{\includegraphics[width=0.95\linewidth]{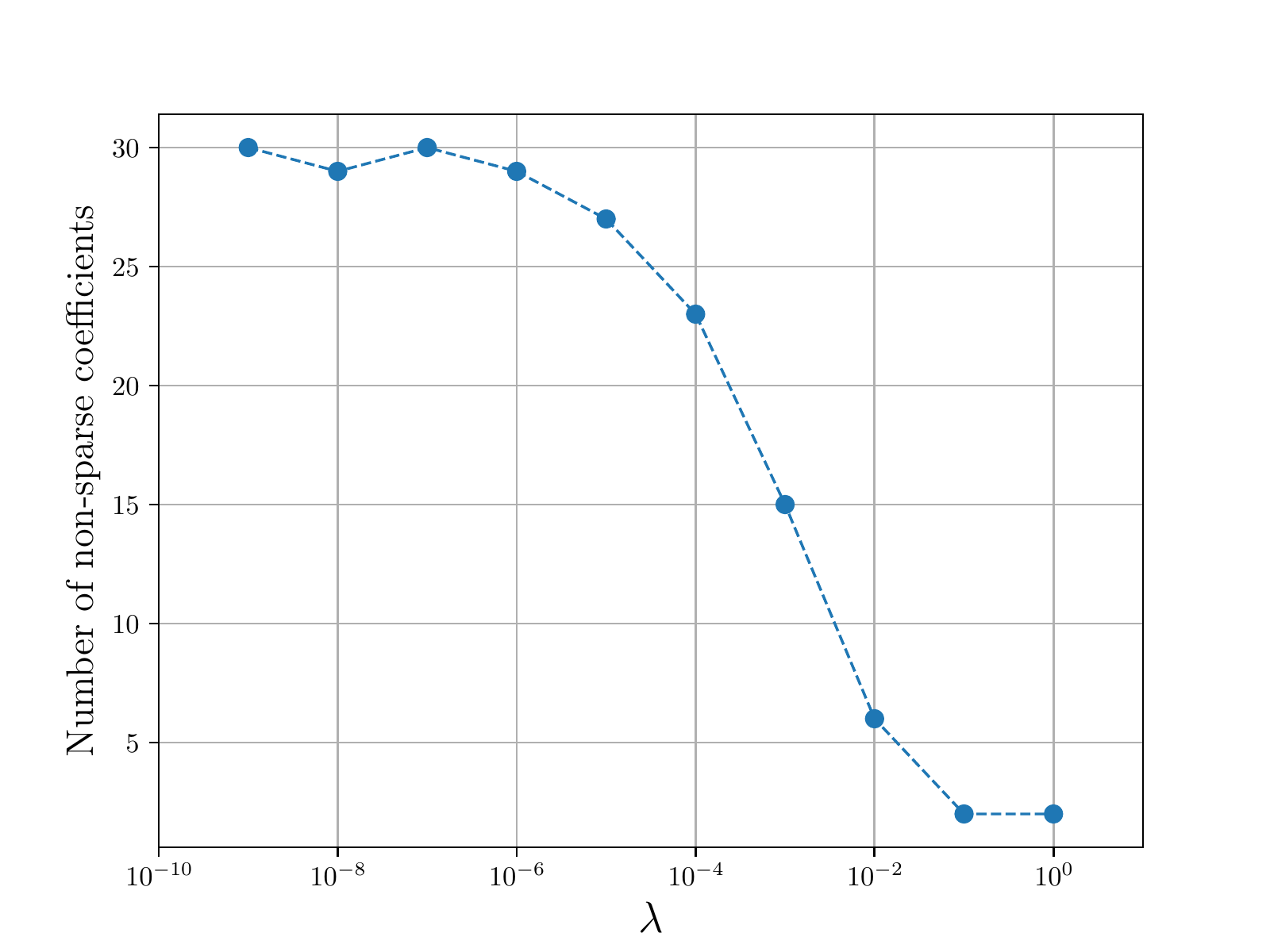}}%  \vspace{2.0cm}
\caption{ From top to bottom: error rate, Lipschitz bound, and  total number of nonzero ReLU coefficients versus $\lambda$ in the simple architecture with  layer descriptor $(2,2,1)$. }
  \label{Fig:VsLambda} \medskip
\end{minipage}
\end{figure}  
As for the error rate, a definite transition occurs as $\lambda$ varies. This suggests a  range of ``proper''  values of $\lambda$ (in this case, $\lambda<10^{-3}$) in which the error would not change much.   The critical value $\lambda=10^{-3}$ is certainly the best choice, since it has a small error and, in addition, the overall network is maximally regularized in the sense of Lipschitz. However, one is required to compute these curves for each value of $\mu$  to find the optimal $\lambda$,  which can be time consuming. A heuristic (but faster) approach  is to honor \eqref{Eq:LmabdaTune}. In this case, it yields $\lambda\approx 0.5 \times 10^{-5}$, which lies within the favourable range of each plot.

\begin{table}[t]
\renewcommand{\arraystretch}{1.3}
\caption{ Number of parameters and Performance in the area classification experiment.}
\label{Table:AreaClassification}
\centering
\begin{tabular}{r|ccc}
\hline\hline
 & Architecture & $N_{\text{param}}$ & Performance \\
 \hline
ReLU & $(2,10,1)$ & 41 & 98.15 \\
LeakyReLU & $(2,10,1)$& 41 & 98.12    \\
PReLU & $(2,10,1)$& 51 & 98.19 \\
Deep Lipschitz  & $(2,2,1)$& {\bf 23} & {\bf 98.54} \\
\hline\hline
\end{tabular}
\end{table}
 
\subsection{$\ell_1$ versus $\ell_2$ Outer-Norms}\label{Subsec:L1vsL2}
As mentioned in Remark \ref{Rem:lpNorm}, we can replace the $\ell_1$ outer norm in our Lipschitz bound by   $\ell_2$,   which results in a tighter bound. We   compare  the two cases in the area classification experiment, where we consider    a network with the layer descriptor $(2,10,1)$ and train it with two outer-norms. The results are reported in Table \ref{Table:L1L2}. As can be seen, the two cases have similar performances. However,  the $\ell_1$ outer-norm results in a much sparser network with fewer parameters, due to its global sparsifying effect. 
\begin{table}[t]
\renewcommand{\arraystretch}{1.3}
\caption{ Effect of the $\ell_1$ vs. $\ell_2$ outer norms in the area classifcation experiment}
\label{Table:L1L2}
\centering
\begin{tabular}{r|ccc}
\hline\hline
 Outer-norm &  $N_{\text{param}}$  & Performance \\
 \hline
$\ell_1$ & 66 & 98.61  \\
$\ell_2$  & 89 & 98.39  \\
\hline\hline
\end{tabular}
\end{table}

\subsection{Effect of the Parameter $K$}
Until now, we have performed all experiments with $K=21$ spline knots. In this section, we let $K$ vary and examine how this effects.  We consider the area classification problem described in Section \ref{Sec:numset} and  train a simple neural network with layer descriptor $(2,2,1)$.

We also perform this experiment on the MNIST dataset \cite{lecun1998gradient} that consists of $28\times 28$ grayscale images of digits from 0 to 9. In this case, we used a neural network that consists of three blocks. The first two  are each composed of three layers: 1) a convolutional layer with a filter of size $5 \times 5$ and two  output channels; 2) a nonlinear layer that has shared activations across each output channels (two activations in each layer are being learned); and 3) a max-pooling layer with kernel and stride of size 2.  The third block is composed of a fully connected layer with output of size 10 followed by soft-max. The output of the network represents the probability of each digit.

The results are depicted in Figure \ref{Fig:KPlots}. In both cases, we also indicated the performance of ReLU and PReLU for comparison. We observe that the performance monotonically increases with $K$ until it reaches saturation. We conclude that, although finding the best value for $K$ is challenging,   suboptimal value still leads to substantial improvements in the performance of the network and typically to better performances than ReLU networks and its variants.
\begin{figure}[t]
\begin{minipage}{\linewidth}
\centering
\centerline{\includegraphics[width=0.95\linewidth]{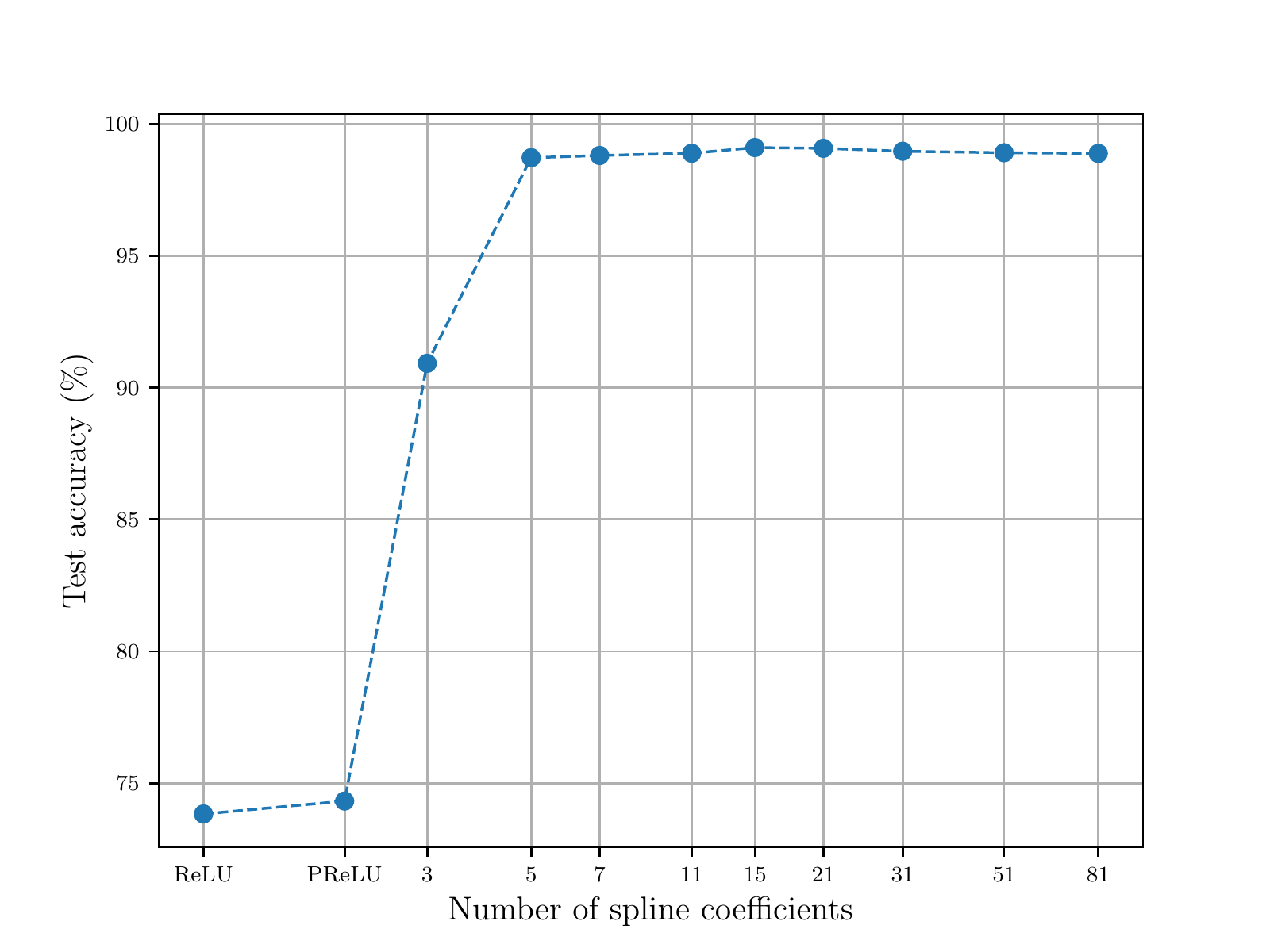}}
\centerline{\includegraphics[width=0.95\linewidth]{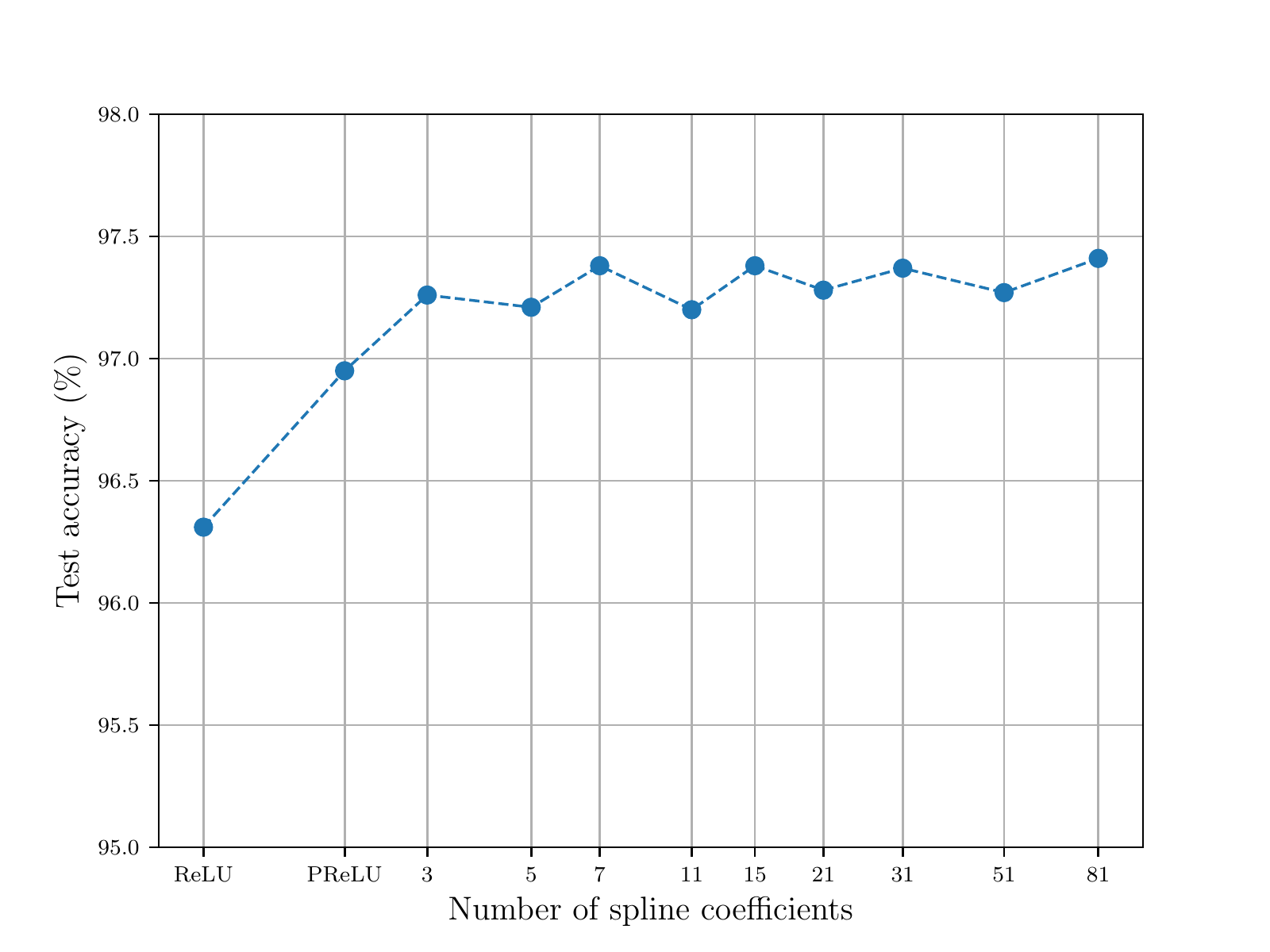}}
\caption{ The performance versus the number $K$ of spline knots of each activation functions in the area classification  (above) and in the MNIST experiment (below). }
  \label{Fig:KPlots} \medskip
\end{minipage}
\end{figure} 
 
\section{Conclusion}
\label{Sec:conclusion}
In this paper, we have introduced a variational framework to learn the activations of a deep neural network while controlling  its global Lipschitz regularity.  We have considered neural networks with   second-order bounded-variation activations and we provided a global bound for their Lipschitz constants. We have showed that the solution of our proposed variational problem exists and is in the form of a deep-spline network with continuous piecewise linear activation functions.  Our future work in this direction is to explore how the simplification of architectures can be compensated by the deployment of more complex activations. 

\section{Acknowledgement}
\label{Sec:Ackno}
The authors would like to thank the anonymous reviewers for their constructive feedbacks that improved the quality of the paper. They also would like to thank Dr. Julien Fageot, Dr. Pol del Aguila Pla, Dr. Jaejoon Yoo and, Pakshal Bohra for fruitful discussions.
\appendix 
\subsection{Topological Structure of $\mathrm{BV}^{(2)}(\mathbb{R})$} \label{App:BV2Space}
For a dual pair $(\mathcal{X},\mathcal{X}')$ of Banach spaces, the sequence $\{w_n\}_{n\in\mathbb{N}}\in\mathcal{X}'$ converges in the weak*-topology to $w_{\lim}\in \mathcal{X}'$  if, for any element $\varphi \in \mathcal{X}$, we have that
\begin{equation}
\langle w_n, \varphi \rangle \rightarrow \langle w_{\lim},\varphi \rangle, \quad n\rightarrow +\infty.
\end{equation}
Consequently, a functional $\nu : \mathcal{X}' \rightarrow\mathbb{R}$ is weak*-continuous  if $\nu(w_n) \rightarrow \nu(w_{\lim})$ for any sequence  $\{w_n\}_{n\in\mathbb{N}}\in\mathcal{X}'$ that converges in the weak*-topology to $w_{\lim}$. This can be shown to be equivalent to the inclusion $\nu \in \mathcal{X}$. In other words, the predual space $\mathcal{X}$ is isometrically isomorphic to the space of weak*-continuous functionals over $\mathcal{X}'$. 

We say that a sequence of neural networks converges in the weak*-topology if 
\begin{itemize}
\item the networks all have the same  layer descriptor (architecture); 
\item for any neuron in the architecture, the corresponding linear weights converge in the Euclidean  topology and the corresponding activations converge   in the weak*-topology of $\mathrm{BV}^{(2)}(\mathbb{R})$.
\end{itemize}     
 We conclude Appendix \ref{App:BV2Space} by mentioning two important properties of the space of second-order bounded-variation functions.
\begin{enumerate}[label=(\roman*)]
\item The sampling functionals defined as $\delta_{x_0}:f \mapsto f(x_0)$ for   $x_0\in\mathbb{R}$ are weak*-continuous in the topology of $\mathrm{BV}^{(2)}(\mathbb{R})$ \cite[Theorem 1]{Unser2018}.  
\item  Any function $f \in {\rm BV}^{(2)}(\R)$ can be uniquely represented as 
\begin{equation}\label{Eq:RepBV}
f(x) = \int_\mathbb{R} h(x,y) u(y) {\rm d}y  + b_1 + b_2 x, 
\end{equation}
where   
$h(x,y) = (x-y)_+ - (1-x) (-y)_+ -x(1-y)_+$, $u= \mathrm{D}^2 f$, $b_1 = f(0)$, and $b_2= f(1)-f(0)$. This result is a special case of Theorem 5 in \cite{unser2017splines}.  
\end{enumerate}

\subsection{Proof of Proposition \ref{BVLipschitz}} \label{App:BVLip}
\begin{proof}
First, one needs to verify that the kernel function $h(\cdot,y_0)$ in \eqref{Eq:RepBV} is Lipschitz-continuous with the constant 1 for all $y_0\in \mathbb{R}$. Indeed, for any $x_1,x_2,y_0\in\mathbb{R}$, we have that 
\begin{equation}
 |h(x_1,y_0) - h(x_2,y_0)| \leq |x_1-x_2|.
\end{equation}
Now, using the  representation \eqref{Eq:RepBV}    together with the triangle inequality, we have that
\begin{align*}
& |f(x_1)-f(x_2)|  = \\ & \left| \int_{\mathbb{R}} \left(h(x_1,y)-h(x_2,y)\right) w(y) \mathrm{d}y +  b_2  (x_1-x_2) \right| \nonumber \\ 
& \leq \left| \int_{\mathbb{R}} \left(h(x_1,y)-h(x_2,y)\right) w(y) \mathrm{d}y \right| 
  + |b_2| |x_1-x_2| \\ & \leq   \int_{\mathbb{R}} \left|h(x_1,y)-h(x_2,y)\right| |w(y)| \mathrm{d}y  
 + |b_2| |x_1-x_2| \\ & \leq   |x_1-x_2| \int_{\mathbb{R}} |w(y)| \mathrm{d}y  
 + |b_2| |x_1-x_2| \\ &  =  |x_1 -x_2| \|w\|_{\mathcal{M}}  + |b_2| |x_1 -x_2|  \\&   \leq \|f\|_{\mathrm{BV}^{(2)}} |x_1 -x_2|.
\end{align*}
\end{proof}
\subsection{Proof of Proposition \ref{Prop:OneSidedDerivatives}} \label{App:OneSideDer}
\begin{proof}
We   prove the existence of the right derivative at $x=0$ and deduce the existence of the left derivative by symmetry. Moreover,  since $\mathrm{BV}^{(2)}(\mathbb{R})$ is a shift-invariant function space, the existence of left and right derivatives will be ensured at any point $x_0\in\mathbb{R}$. 

Let us denote 
\begin{equation}
\Delta\sigma(a,b) =\frac{\sigma(a)-\sigma(b)}{a-b}, \quad a,b\in \mathbb{R}, a\neq b. 
\end{equation}
From Proposition \ref{BVLipschitz}, we have that
\begin{equation}\label{FinitenessDelta}
-\|\sigma\|_{\mathrm{BV}^{(2)}} \leq \Delta\sigma(a,b)  \leq \|\sigma\|_{\mathrm{BV}^{(2)}}, \quad  a,b\in \mathbb{R}, a\neq b.
\end{equation}
 Define quantities  $M_{\sup}$ and $M_{\inf}$ as  
\begin{align*}
& M_{\sup} = \limsup_{h\rightarrow 0^+} \Delta\sigma(h,0), \\ & M_{\inf} = \liminf_{h\rightarrow 0^+} \Delta\sigma(h,0).
\end{align*}
The finiteness of $|M_{\sup}|$ and $|M_{\inf}|$ is guaranteed by 
\eqref{FinitenessDelta}. Now, it remains to  show that $M_{\sup} = M_{\inf}$ to prove the existence of the right derivative.  Assume, by contradiction, that $M_{\sup} > M_{\inf}$. Consider a small value of $0<\epsilon<\frac{M_{\sup}-M_{\inf}}{3}$ and define the constants $C_1 = \left(M_{\sup} - \epsilon\right)$ and $C_2 =\left(M_{\inf} + \epsilon\right)$. Clearly, we have $C_1-C_2\geq \epsilon>0$. Moreover, due to the definition of $\limsup$ and $\liminf$, there exist  sequences $\{a_n\}_{n=0}^{\infty}$ and $\{b_n\}_{n=0}^{\infty}$ that are monotically decreasing to 0 and are such that 
\begin{align*}
& \Delta\sigma(a_n,0) > C_1, \quad \Delta\sigma(b_n,0)< C_2, \\
&  a_n > b_n > a_{n+1}, \quad \forall n\in\mathbb{N}.
\end{align*}
One then has that
\begin{align*}
\sigma(a_n)> C_1 a_n+ \sigma(0), \quad  \sigma(b_n)< C_2 b_n + \sigma(0)
\end{align*}
and, consequently, 
\begin{align}
& \Delta\sigma(a_n,b_n) \geq \frac{ C_1 a_n - C_2 b_n }{a_n-b_n}> C_1   , \label{Eq:nn}\\   
& \Delta\sigma(b_n,a_{n+1}) \leq \frac{C_2 b_n -C_1 a_{n+1}}{b_n-a_{n+1}} <C_2. \label{Eq:nnplus1}
\end{align}
From the definition of second-order total variation and using  \eqref{Eq:nn} and \eqref{Eq:nnplus1}, we obtain that
\begin{align*}
\|\mathrm{D}^2 \sigma\|_{\mathcal{M}} &\geq \sum_{n=0}^{\infty} \left| \Delta\sigma(a_n,b_n) -\Delta\sigma(b_n,a_{n+1})   \right|
\\ &  \geq \sum_{n=0}^{\infty} \left( \frac{ C_1 a_n - C_2 b_n }{a_n-b_n} - \frac{C_2 b_n -C_1 a_{n+1}}{b_n-a_{n+1}}\right) \\ & \geq \sum_{n=1}^{\infty} (C_1 -C_2 ) = \sum_{n=0}^\infty \epsilon =+\infty, 
\end{align*}
which contradicts the original assumption $\sigma\in\mathrm{BV}^{(2)}(\mathbb{R})$. Hence, $M_{\sup}=M_{\inf}$ and the right derivative exists.
 \end{proof}
\subsection{Proof of Theorem \ref{Thm:LipschitzBV}} \label{App:NNLip}
\begin{proof}
From Proposition \ref{BVLipschitz},  for any $\boldsymbol{x}_1,\boldsymbol{x}_2\in\mathbb{R}^{N_{l-1}}$, we have that
\begin{align}\label{Eq:Prop1Neuron}
\left|\sigma_{n,l}(\M w_{n,l}^T\V x_1)-\sigma_{n,l}(\M w_{n,l}^T\V x_2)\right|\leq \|\sigma_{n,l}\|_{{\rm BV}^{(2)}} \left|\M w_{n,l}^T(\V x_1-\V x_2)\right|.
\end{align}
Now, by using     Hölder's inequality, we   bound the Lipschitz constant of the linear layers as
\begin{equation}\label{Holder}
 \left|\M w_{n,l}^T(\V x_1-\V x_2)\right| \leq \|\M w_{n,l} \|_q  \|\V x_1-\V x_2\|_p. 
\end{equation}
By combining \eqref{Eq:Prop1Neuron} and \eqref{Holder} and using the fact that  $\|\mathbf{w}_{n,l}\|_{q}\leq \|\mathbf{W}_l\|_{q,\infty}$, we obtain that
\begin{align}
& \left|\sigma_{n,l}(\M w_{n,l}^T\V x_1)-\sigma_{n,l}(\M w_{n,l}^T\V x_2)\right|^p   \nonumber \\ & \quad \leq \|\mathbf{W}_l\|_{q,\infty}^p \|\sigma_{n,l}\|_{{\rm BV}^{(2)}}^p  \|\V x_1-\V x_2\|_p^p,  
\end{align}
which is a Lipschitz bound for the $(n,l)$th neuron of the neural network. By summing up over the neurons of layer $l$, we control the output of this layer as 
\begin{align}\label{Eq:LayerLipschitz}
\left\| \M f_{l}(\V x_1) - \M f_{l}(\V x_2)\right\|_p  \le  \|\mathbf{W}_{ l}\|_{q,\infty} \|\boldsymbol{\sigma}_l\|_{\mathrm{BV}^{(2)},p}  \|\V x_1-\V x_2\|_p.
\end{align}
Now,  the composition of the layer inequalities results in the inequality \eqref{LipschitzFinalBound} with the constant introduced in \eqref{LipschitzConstant}.
\end{proof}

\subsection{Proof of Proposition \ref{Prop:LipEuc}}\label{App:LipEuc}
Following Proposition \ref{BVLipschitz} and using Cauchy–Schwarz' inequality, we obtain that 
\begin{align}
& \left|\sigma_{n,l}(\M w_{n,l}^T\V x_1)-\sigma_{n,l}(\M w_{n,l}^T\V x_2)\right|  \nonumber 
\\& \quad \leq  \|\sigma_{n,l}\|_{{\rm BV}^{(2)}} \|\M w_{n,l}\|_{2}\|\V x_1-\V x_2\|_2.
\end{align}
Combining it with the known  hierarchy between the discrete $\ell_p$ norms and, in particular, the inequality $\|\V x\|_{2} \leq \|\V x\|_{1}$ for any $\V x\in \mathbb{R}^{N_l}$, we obtain that 
\begin{align}
\left\| \M f_{l}(\V x_1) - \M f_{l}(\V x_2)\right\|_2   &\leq \left\| \M f_{l}(\V x_1) - \M f_{l}(\V x_2)\right\|_1  \nonumber \\ & =\sum_{n=1}^{N_l}  \left|\sigma_{n,l}(\M w_{n,l}^T\V x_1)-\sigma_{n,l}(\M w_{n,l}^T\V x_2)\right|  \nonumber \\&  \leq \sum_{n=1}^{N_l}  \|\sigma_{n,l}\|_{{\rm BV}^{(2)}}  \|\M w_{n,l}\|_{2} \|\V x_1-\V x_2\|_2  \nonumber\\& \leq  \|\V \sigma_{l}\|_{{\rm BV}^{(2)},2}  \|\M W_l\|_F  \|\V x_1-\V x_2\|_2^2 \label{Eq:BoundEuclid}.
\end{align}
Note that in the last inequality of \eqref{Eq:BoundEuclid}, we have again used Cauchy-Schwarz' inequality. Combining with $\|\V \sigma_{l}\|_{{\rm BV}^{(2)},2}\leq \|\V \sigma_{l}\|_{{\rm BV}^{(2)},1}$, we have that 
\begin{align} \label{Eq:BoundEuclid2}
 \left\| \M f_{l}(\V x_1) - \M f_{l}(\V x_2)\right\|_2      \leq  \|\V \sigma_{l}\|_{{\rm BV}^{(2)},1}  \|\M W_l\|_F \|\V x_1-\V x_2\|_2.
\end{align}
Finally by composing \eqref{Eq:BoundEuclid2} through the layers, we obtain the announced bound. 

\subsection{Proof of Theorem \ref{Thm:weakstarNN}} \label{App:NNweakstar}
We first prove Lemma \ref{lemma:Shayan}. We recall that a sequence of functions $f_t:\mathcal{X}\rightarrow\mathcal{Y}$, $t\in\mathbb{N}$,  converges pointwise to $f_{\lim}:\mathcal{X}\rightarrow\mathcal{Y}$  if, for all  $x\in \mathcal{X}$, 
\begin{equation}
f_{\lim}(x) = \lim_{t\rightarrow +\infty} f_t(x).
\end{equation}
\begin{lemma}\label{lemma:Shayan}
Given the Banach spaces $\mathcal{X}$,$\mathcal{Y}$, and $\mathcal{Z}$, consider the two sequences of functions $f_t:\mathcal{X}\rightarrow\mathcal{Y}$ and  $g_t:\mathcal{Y}\rightarrow\mathcal{Z}$ such that they converge pointwise to the functions $f_{\lim}:\mathcal{X}\rightarrow\mathcal{Y}$ and $g_{\lim}:\mathcal{Y}\rightarrow\mathcal{Z}$, respectively.  Moreover, assume that the functions $g_t$ are all Lipschitz-continuous with a shared constant $C>0$, so that,  for any $y_1,y_2\in\mathcal{Y}$, one has that
\begin{equation}
\|g_t(y_1)-g_t(y_2)\|_\mathcal{Z} \leq C\|y_1-y_2\|_\mathcal{Y}, \quad \forall t\in\mathbb{N}.
\end{equation}
Then, the composed sequence  $h_t:\mathcal{X}\rightarrow\mathcal{Z}$ with $h_t = g_t \circ f_t$ converges pointwise to $h_{\lim}=g_{\lim} \circ f_{\lim}$.
\end{lemma}
\begin{proof}
We use the triangle inequality to obtain  that
\begin{align}\label{Eq:Triangle}
\|h_t(x) - h_{\lim}(x) \|_\mathcal{Z} &\leq \|h_t(x) - g_{t}(f_{\lim}(x))\|_\mathcal{Z} \nonumber \\ & \quad +\|g_{t}(f_{\lim}(x)) - h_{\lim}(x)\|_\mathcal{Z} 
\end{align} 
for all $x\in\mathcal{X}$. The uniform Lipschitz-continuity of $g_t$ then yields that
\begin{align}
\|h_t(x) - g_{t}(f_{\lim}(x))\|_\mathcal{Z} & = \|g_t(f_t(x)) - g_{t}(f_{\lim}(x))\|_\mathcal{Z} \nonumber \\ & \quad \leq C\|f_t(x) - f_{\lim}(x)\|_\mathcal{Y} \rightarrow 0  \label{Eq:pointwise}
\end{align}
as $t\rightarrow +\infty$. This is due to  the pointwise convergence of $\{f_t\}\rightarrow f_{\lim}$. Similarly, the pointwise convergence $\{g_t\}\rightarrow g_{\lim}$ implies that  
\begin{align*}
  \|g_{t}(f_{\lim}(x)) - g_{\lim}(f_{\lim}(x))\|_\mathcal{Z} \rightarrow 0, \quad t\rightarrow +\infty
\end{align*}
which, together with \eqref{Eq:pointwise} and \eqref{Eq:Triangle}, proves the  pointwise convergence of  $h_t\rightarrow h_{\lim}$ as $t\rightarrow +\infty$.
\end{proof}

\begin{proof}[Proof of Theorem \ref{Thm:weakstarNN}]
Assume that the sequence   $\{\mathbf{f}_{\rm deep}^{(t)}\}$ of neural networks with layers $\mathbf{f}_l^{(t)}=\boldsymbol{\sigma}_l^{(t)}\circ \mathbf{W}_l^{(t)}:\mathbb{R}^{N_{l-1}}\rightarrow\mathbb{R}^{N_{l}}$ for $l=1,\ldots,L$ (described in \eqref{Eq:DNN}) converges in the weak*-topology to 
\begin{align*}
f_{\rm deep}^{\lim}=\mathbf{f}_L^{\lim} \circ \cdots \circ \mathbf{f}_1^{\lim}, \quad \forall l: \mathbf{f}_{l}^{\lim}=\boldsymbol{\sigma}_{l}^{\lim}\circ \mathbf{W}_{l}^{\lim}.
\end{align*}
 By definition, every element of $\boldsymbol{\sigma}_{l}^{(t)}$ converges in the weak*-topology to the corresponding element in $\boldsymbol{\sigma}_{l}^{\lim}$. The convergence is also pointwise due  to the weak*-continuity of the sampling functional in the space of activation functions $\mathrm{BV}^{(2)}(\mathbb{R})$.  The conclusion is thaty  $\mathbf{f}_{l}^{(t)}$  also converges pointwise   to $\mathbf{f}_{l}^{\lim}$.
 
 In addition,  knowing that any norm is weak*-continuous in its corresponding Banach space, one can find the uniform constant $T_1>0$ such that, for all $t>T_1$ and for all $l=1,\ldots,L$, we have that 
 \begin{equation}\label{Eq:UnifSigma}
 \|\boldsymbol{\sigma}_{l}^{(t)}\|_{\mathrm{BV}^{(2)},2} \leq C_1=2 \max_{l}\|\boldsymbol{\sigma}_{l}^{\lim}\|_{\mathrm{BV}^{(2)},2}.
\end{equation}
Similarly, from the convergence $\mathbf{W}_{l}^{(t)} \rightarrow \mathbf{W}_{l}^{\lim}$, one deduces that there exists a constant $T_2>0$ such that, for all $t>T_2$ and for all $l=1,\ldots,L$, we have that 
\begin{equation}\label{Eq:UnifWei}
 \|\mathbf{W}_{l}^{(t)}\|_{2,\infty} \leq C_2=2 \max_{l}\|\mathbf{W}_{l}^{\lim}\|_{2,\infty}.
\end{equation}
Now, from \eqref{Eq:UnifWei} with $p=q=2$ and using \eqref{Eq:LayerLipschitz}, \eqref{Eq:UnifSigma}, we deduce that, for $t>\max(T_1,T_2)$,  each layer of $\mathbf{f}_{\rm deep}^{(t)}$ is Lipschitz-continuous with the shared constant $C=C_1C_2$.  Combining it with the  pointwise convergence  $\mathbf{f}_{l}^{(t)} {\rightarrow}\mathbf{f}_{l}^{\lim}$,  one completes the proof by sequentially using the outcome of Lemma \ref{lemma:Shayan}. 
\end{proof}

\subsection{Proof of Theorem \ref{Thm:Main}} \label{App:Main}
\begin{proof}
We divide the proof in two parts. First, we show the existence of the solution  of \eqref{Eq:DeepCost} and, then,  we show the existence of a solution with activations of the form \eqref{Eq:ActivationDeepSpline}. 

\textbf{ Existence of Solution} Consider an arbitrary neural network $\mathbf{f}_{\rm deep,0}$  with the cost $A= \mathcal{J}(\mathbf{f}_{\rm deep,0};X,Y)$.  The coercivity  of $\mathrm{R}_{l}$  guarantees   the existence of the constants $B_{l}$ for $l=1,2,\ldots,L$ such that 
\begin{equation}
\|\mathbf{w}_{n,l}\|_2 \geq B_{l} \Rightarrow \mathrm{R}_{l}(\mathbf{w}_{n,l}) \geq \frac{A}{\mu_{l}}.
\end{equation}
 This allows us to transform the  unconstrained problem \eqref{Eq:DeepCost} into the   equivalent  constrained minimization 
\begin{align}\label{Eq:DeepCostConst}
 \min_{\substack{ \mathbf{w}_{n,l}\in\mathbb{R}^{N_{l-1}}, \\ \sigma_{n,l} \in \mathrm{BV}^{(2)}(\mathbb{R}) }}
\mathcal{J}(\mathbf{f}_{\rm deep}), \quad \text{s.t.} \quad \begin{cases} 
\| \mathbf{w}_{n,l} \|_2 \leq B_{l}, \\ 
\mathrm{TV}^{(2)}(\sigma_{n,l}) \leq A/ \lambda_{l}, \\ 
|\sigma_{n,l}(0)|  \leq A/\lambda_{l}, \\
| \sigma_{n,l}(1)| \leq A/\lambda_{l}.
 \end{cases}
\end{align}
The equivalence is due to the fact that any neural network that does not satisfy the constraints of \eqref{Eq:DeepCostConst2} has a strictly bigger cost   than $\mathbf{f}_{\rm deep,0}$ and, hence, is not in the solution set.    Due to the decomposition \eqref{Eq:RepBV}, we can  rewrite \eqref{Eq:DeepCostConst} as 
\begin{align}\label{Eq:DeepCostConst2}
 \min_{\substack{ \mathbf{w}_{n,l}\in\mathbb{R}^{N_{l-1}}, \\ u_{n,l} \in  \mathcal{M}(\mathbb{R}) \\ b_{\cdot,n,l}\in\mathbb{R} }}
\mathcal{J}(\mathbf{f}_{\rm deep}), \quad \text{s.t.}, \quad \begin{cases} 
\|\mathbf{w}_{n,l} \|_2 \leq B_l , \\ 
\|u_{n,l}\|_{\mathcal{M}} \leq A/ \lambda_{l}, \\ 
|b_{1,n,l}| ,|b_{2,n,l} | \leq 2A/\lambda_{l}.
 \end{cases}
\end{align}
 Due to the  Banach-Anaoglu theorem \cite{rudin1991functional},  the feasible set in \eqref{Eq:DeepCostConst2} is weak*-compact. Moreover, the cost functional defined in \eqref{Eq:DeepCost} is a composition and sum of lower-semicontinuous functions and weak*-continuous functionals (see Theorem \ref{Thm:weakstarNN}). Hence, it is itself weak*-lower semicontinuous. This guarantees   the existence of a minimizer of \eqref{Eq:DeepCostConst2} (and, consequently, of \eqref{Pb:DeepSpline}), due to the generalized Weierstrass theorem   \cite{kurdila2006convex}.
 
\textbf{Optimal Activations} Let   $\tilde{\mathbf{f}}_{\rm deep}$  be  a  solution of \eqref{Pb:DeepSpline} with  
\begin{equation}
\tilde{\mathbf{f}}_{\rm deep} = \tilde{\mathbf{f}}_L \circ \cdots \circ \tilde{\mathbf{f}}_1, \quad \forall l: \tilde{\mathbf{f}}_{l} = \tilde{\boldsymbol{\sigma}}_{l} \circ \tilde{\mathbf{W}}_{l}.
\end{equation}
For any input vector $\boldsymbol{x}_m$ in the dataset $\boldsymbol{X}$, we then define the vectors $\boldsymbol{z}_{l,m} = (z_{1,l,m},\ldots,z_{N_l,l,m}),\boldsymbol{s}_{l,m} = (s_{1,l,m},\ldots,s_{N_l,l,m})\in \mathbb{R}^{N_l}$ as 
\begin{align*}
&\boldsymbol{z}_{l,m}  = \tilde{\mathbf{f}}_l \circ \cdots \circ \tilde{\mathbf{f}}_1(\boldsymbol{x}_m),\\
&\boldsymbol{s}_{l,m}  = \tilde{\mathbf{W}}_l \circ \tilde{\mathbf{f}}_{l-1} \circ \cdots \circ \tilde{\mathbf{f}}_1(\boldsymbol{x}_m).
\end{align*}
Now, we show that the activation $\tilde{\sigma}_{n,l}$  of the neuron indexed by $(n,l)$ is indeed a solution  of the minimization
\begin{align}
 \min_{   \sigma \in {\rm BV}^{(2)}(\R)  }
    & {\rm TV}^{(2)}(\sigma) =\|\mathrm{D}^2 \sigma \|_{\mathcal{M}}  \quad s.t. \nonumber \\ & \begin{cases} \sigma( s_{n,l,m})=  z_{n,l,m}, & m=1,2,\ldots,M, \\ \sigma(x) = \tilde{\sigma}_{n,l}(x), & x\in \{0,1\} \label{FeasibleCond}.  \end{cases}
\end{align}
Assume by contradiction that there exists a function $\sigma\in{\rm BV}^{(2)}(\R) $ that satisfies the feasiblity conditions \eqref{FeasibleCond} and is  such that ${\rm TV}^{(2)}(\sigma) < {\rm TV}^{(2)}(\tilde{\sigma}_{n,l})$. Then, we have that 
\begin{align}
\|\sigma\|_{\mathrm{BV}^{(2)}} & = {\rm TV}^{(2)}(\sigma)  + |\sigma(0)|+|\sigma(1) | \nonumber \\
&=  {\rm TV}^{(2)}(\sigma)  + |\tilde{\sigma}_{n,l}(0)|+|\tilde{\sigma}_{n,l}(1) | \nonumber  \\ 
& < {\rm TV}^{(2)}(\tilde{\sigma}_{n,l})  + |\tilde{\sigma}_{n,l}(0)|+|\tilde{\sigma}_{n,l}(1) | 
\nonumber \\& = \|\tilde{\sigma}_{n,l}\|_{\mathrm{BV}^{(2)}}.\label{Eq:BVInopt}
\end{align}
In addition, due to the feasiblity assumptions $ \sigma( s_{n,l,m})=  z_{n,l,m}$ for $m=1,\ldots,M$, one 
readily verifies that, by replacing $\tilde{\sigma}_{n,l}$ by $\sigma$ in the optimal neural network $
\tilde{\mathbf{f}}_{\rm deep}$, the data fidelity term $ \sum_{m=1}^M \mathrm{E}\big(\V y_m,\tilde{\M f}_{\rm 
deep}(\V x_m)\big) $ in \eqref{Eq:DeepCost} remains untouched. The same holds for the weight 
regularization term $\sum_{l=1}^L \mathrm{R}_{l} (\mathbf{W}_l)$. However, from 
\eqref{Eq:BVInopt},  one gets a strictly smaller overall $\mathrm{BV}^{(2)}$ penalty with $\sigma$  that  contradicts the 
optimality of $\tilde{\mathbf{f}}_{\rm deep}$.  With a similar argument, one sees that, for 
any solution $\sigma\in {\mathrm{BV}^{(2)}}$ of \eqref{FeasibleCond}, the substitution of $
\tilde{\sigma}_{n,l}$ by $\sigma$ yields another solution of   \eqref{Pb:DeepSpline}. Due to 
Lemma 1 of \cite{Unser2018}, Problem \eqref{FeasibleCond} has a solution that is a linear spline of the form \eqref{Eq:ActivationDeepSpline} with $K_{n,l}\leq \left(\tilde{M}-2\right)$, where $\tilde{M}=M+2$ is the number of constraints in \eqref{FeasibleCond}. By using this result for every neuron $(n,l)$, we verify  the existence of a deep-spline solution of \eqref{Pb:DeepSpline}. 
\end{proof}

\subsection{Proof of Theorem \ref{Thm:BVProp}} \label{App:BVProp}
\begin{proof}
For any local minima $\mathbf{f}_{\rm deep}$ of \eqref{Pb:DeepSpline} with linear weights $\mathbf{W}_{l}$ and nonlinear layers $\boldsymbol{\sigma}_{l}$ and for any layer $l^*\neq L$, consider the perturbed network $\mathbf{f}_{\rm deep,\epsilon}$ with  the linear layers 
\begin{equation}
\mathbf{W}_{l,\epsilon}  = \begin{cases} \mathbf{W}_l, & l\neq l^*+1\\ 
  (1+\epsilon) \mathbf{W}_{l^*}, & l=l^*+1 \end{cases}
\end{equation}
 and the nonlinear layers 
\begin{equation}
\boldsymbol{\sigma}_{l,\epsilon}  = \begin{cases}   \boldsymbol{\sigma}_{l}, & l\neq l^* \\ (1+\epsilon)^{-1} \boldsymbol{\sigma}_{l^*}, & l=l^*
 \end{cases}
\end{equation} 
for any $\epsilon \in (-1,1)$.  One readily verifies that, for any $\boldsymbol{x} \in\mathbb{R}^{N_0}$ and any $\epsilon \in \mathbb{R}$, we have that $\mathbf{f}_{\rm deep,\epsilon} (\boldsymbol{x}) = \mathbf{f}_{\rm deep}(\boldsymbol{x})$ and, hence, both networks have the same data-fidelity penalty in the global cost \eqref{Eq:DeepCost}. In fact, the only difference between their overall cost is associated to the regularization terms of the $(l^*+1)$th linear layer and the $l$th nonilnear layer. For those,   the scaling property of norms yields that
\begin{align}
&\| \mathbf{W}_{l^*+1, \epsilon}\|_F^2 = (1+\epsilon)^2 \| \mathbf{W}_{l^*+1}\|_F^2, \label{Eq:LinearPerturbed}\\ &  \|\boldsymbol{\sigma}_{l^*,\epsilon}\|_{\mathrm{BV}^{(2)},1}=(1+\epsilon)^{-1}\| \boldsymbol{\sigma}_{l^*}\|_{\mathrm{BV}^{(2)},1}.   
\label{Eq:NonlinearPerturbed}
\end{align}
Due to the (local) optimality of $\mathbf{f}_{\rm deep}$, there exists a constant $\epsilon_{\max}$ such that, for all $\epsilon \in (-\epsilon_{\max},\epsilon_{\max})$, we have that
\begin{equation}\label{Eq:OptimalityLocal}
\mathcal{J}(\mathbf{f}_{\rm deep}) \leq \mathcal{J}(\mathbf{f}_{\rm deep,\epsilon}).
\end{equation}
Now, from  \eqref{Eq:LinearPerturbed} and \eqref{Eq:NonlinearPerturbed}, we have that 
\begin{align*}
\mu_{l^*+1}\| \mathbf{W}_{l^*+1}\|_F^2+ \lambda_{l}\| \boldsymbol{\sigma}_{l^*}\|_{\mathrm{BV}^{(2)},1} & \leq  \mu_{l^*+1} (1+\epsilon)^2 \| \mathbf{W}_{l^*+1}\|_F^2\\& \quad +  \lambda_{l }(1+\epsilon)^{-1}\| \boldsymbol{\sigma}_{l^*}\|_{\mathrm{BV}^{(2)},1},
\end{align*}
for any $\epsilon \in  (-\epsilon_{\max},\epsilon_{\max})$. By simplifying the latter inequality, we get that 
\begin{equation}
0 \leq \epsilon g(\epsilon), \quad \forall \epsilon \in (-\epsilon_{\max},\epsilon_{\max}), 
\end{equation}
where $g(\epsilon)= \mu_{l^*+1} \|\mathbf{W}_{l^*+1}\|_F^2 (\epsilon+2) - \lambda_{l} \|\boldsymbol{\sigma}\|_{\mathrm{BV}^{(2)}} (1+\epsilon)^{-1} $ is a continuous function of $\epsilon$ in the interval $(-1,1)$. This yields that $g(\epsilon)$ is nonnegative for positive values of $\epsilon$ and is nonpositive for negative values of $\epsilon$. Hence, we get that $g(0)=0$ and, consequently, that 
\begin{equation}
\lambda_{l^*} \| \boldsymbol{\sigma}_{l^*}\|_{\mathrm{BV}^{(2)}} = 2 \mu_{l^*+1} \| \mathbf{W}_{l^*+1}\|_F^2.
\end{equation}
\end{proof}

 \bibliographystyle{IEEEtran}
\bibliography{Deep.bib}

\end{document}